\documentclass{article}

\usepackage{graphicx}
\usepackage{float}
\usepackage{caption}
\usepackage{subcaption}

\usepackage[super]{nth}
\usepackage{amsmath}
\usepackage{bm}
\usepackage{hyperref}

\usepackage{placeins}
\usepackage{array}
\usepackage{tabu}
\usepackage{amssymb}
\usepackage{mathtools}
\usepackage{lineno}

\usepackage{wrapfig}
\usepackage{booktabs}       
\usepackage{float}
\usepackage{authblk}
\usepackage{enumerate}
\usepackage{xcolor}

\usepackage{bbm}
\usepackage{amsthm}
\newtheorem{theorem}{Theorem}
\newtheorem{lemma}{Lemma}

\title{Generative Ensemble Regression: Learning Particle Dynamics from Observations of Ensembles with Physics-Informed Deep Generative Models}

\author[1]{Liu Yang}
\author[2]{Constantinos Daskalakis}
\author[1,3,*]{George Em Karniadakis}
\affil[1]{Division of Applied Mathematics, Brown University, Providence, RI 02912, USA}
\affil[2]{Department of Electrical Engineering and Computer Science, Massachusetts Institute of Technology, Cambridge, Massachusetts, USA}
\affil[3]{Pacific Northwest National Laboratory, Richland, WA 99354, USA}
\affil[*]{Correspondence author, george\_karniadakis@brown.edu}
\date{\vspace{-5ex}}

\begin{document}

\maketitle

\begin{abstract}
We propose a new method for inferring the governing stochastic ordinary differential equations (SODEs) by observing particle ensembles at discrete and sparse time instants, i.e., multiple ``snapshots''. Particle coordinates at a single time instant, possibly noisy or truncated, are recorded in each snapshot but are unpaired across the snapshots. By training a physics-informed generative model that generates ``fake'' sample paths, we aim to fit the observed particle ensemble distributions with a curve in the probability measure space, which is induced from the inferred particle dynamics. We employ different metrics to quantify the differences between distributions, e.g., the sliced Wasserstein distances and the adversarial losses in generative adversarial networks (GANs). We refer to this method as generative ``ensemble-regression'' (GER), in analogy to the classic ``point-regression'', where we infer the dynamics by performing regression in the Euclidean space. We illustrate the GER by learning the drift and diffusion terms of particle ensembles governed by SODEs with Brownian motions and L\'evy processes up to 100 dimensions. We also discuss how to treat cases with noisy or truncated observations. Apart from systems consisting of independent particles, we also tackle nonlocal interacting particle systems with unknown interaction potential parameters by constructing a physics-informed loss function. Finally, we investigate scenarios of paired observations and discuss how to reduce the dimensionality in such cases by proving a convergence theorem that provides theoretical support. 

\end{abstract}

\section*{Introduction}
Classic methods for inferring the ordinary differential equation (ODE) dynamics from data usually require observations of a point or particle governed by the ODE at different time instants. We refer this learning paradigm as ``point-regression''. More specifically, as illustrated in Figure~\ref{fig:schematic0}, in point-regression problems we aim to infer the governing ODE of a point and perhaps also the initial condition, given the (possibly noisy) observations of its coordinates at different time instants. Typically, we optimize the dynamics and the initial coordinate so that the inferred curve matches the data in the Euclidean distance. Let us consider, for simplicity, the one-dimensional linear regression with quadratic loss as an example. Given observations of $x$ at multiple $t$, we want to optimize the parameter $a$ in the ODE $dx_t/dt = a$ as well as the initial point $x_0 = b$ so that the mean squared $L_2$ distance between predictions and data points is minimized, where $a$ and $b$ are the slope and the intercept of the linear function. Other examples in this category include logistic regression, recurrent neural networks~\cite{hochreiter1997long} and the neural ODE~\cite{chen2018neural} for time series, etc.

For systems consisting of an ensemble of particles, point-regression may fail to apply. For example, we want to infer the governing stochastic ordinary differential equations (SODE) from observations of particle ensembles at discrete and sparse time instants, but the data of an individual particle are not sufficiently informative for dynamic inference. Another example is a system consisting of a large number of interacting particles, even close to the mean field limit, e.g., we may want to infer how the fish interact with each other from discrete snapshots of the fish school. In such scenarios, instead of learning from individual particles, we need to {\em learn from the particle ensembles}. Specifically, we wish to infer the governing dynamics and perhaps also the initial condition, using observations of an ensemble of particle at discrete time instants. We call an observation at a single time instant a ``snapshot'', where part or all of the particle coordinates in the ensemble are recorded. Since the particles could be indistinguishable in observations, especially in a large system, we thus consider the case where the data are not labeled with particle indices, in other words, we cannot pair data across snapshots.

We call this paradigm ``ensemble-regression'' in analogy to the ``point-regression''. As illustrated in Figure~\ref{fig:schematic0}, the initial condition and dynamics for particles would induce a curve $t \rightarrow \rho_t$ in the probability measure space, where $\rho_t$ denotes the particle distributions at time $t$. Such $\rho_t$ will be governed by a corresponding partial differential equation (PDE), e.g., the Fokker-Planck equation if the particles are governed by SODEs of diffusion processes. We aim to optimize the dynamics and the initial distribution so that the inferred curve matches the distributions from the data, and the differences can be quantified with certain metrics.

\begin{figure}[H]
\centering
\includegraphics[width = 0.6\textwidth]{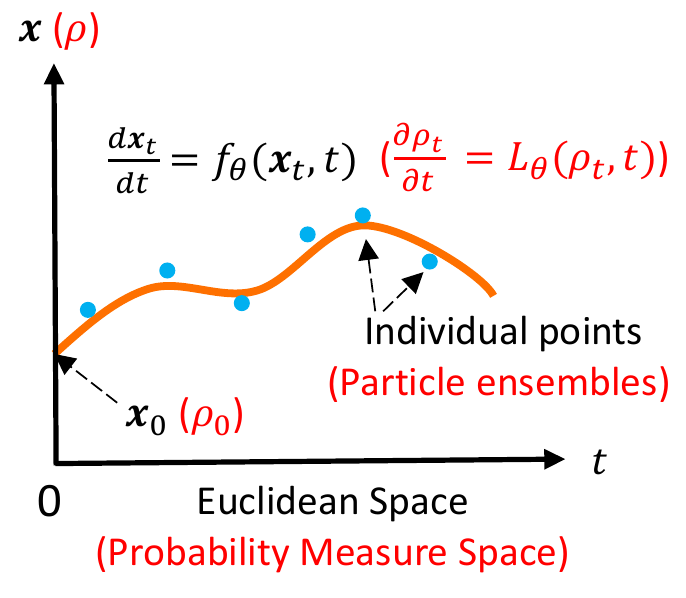}
\includegraphics[width = 0.8\textwidth]{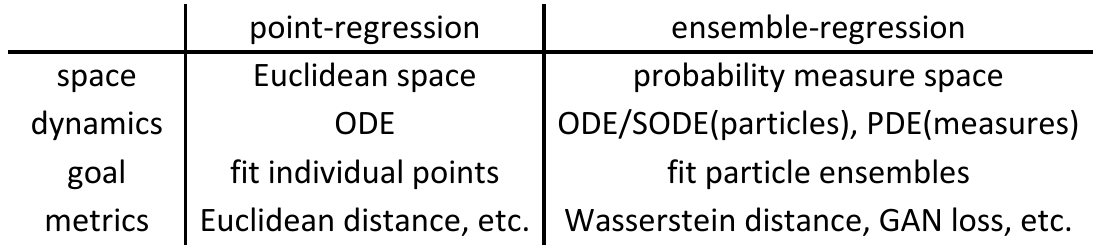}
\caption{Schematic showing the two paradigms of point-regression and ensemble-regression for dynamic inference. In point-regression problems (black labels), we aim to fit the point coordinates (blue dots) from data with the inferred curve (orange curve), determined by the initial coordinate $\bm{x}_0$ as well as the ODE $d\bm{x}_t/dt = f_{\theta}(\bm{x}_t, t)$, where $f_{\theta}$ is a function. As an analogue, in ensemble-regression problems (red labels), we aim to fit the distributions of the ensemble in the snapshots (blue dots) with the inferred curve (orange curve), determined by the initial distribution $\rho_0$ as well as the PDE $\partial\rho_t/\partial t = L_{\theta}(\rho_t, t)$, where $L_{\theta}$ is an operator.}
 \label{fig:schematic0}
\end{figure}

Herein we propose a new method to perform ensemble-regression. We use a generative model with deep neural networks as build blocks, which will generate ``fake'' particle systems, to represent the inferred curve in the probability measure space, and then perform regression in the probability measure space with the inferred curve. We thus name our method as {\em generative ensemble-regression}.  In this paper we test the sliced Wasserstein (SW) distance~\cite{deshpande2018generative} and the loss in generative adversarial networks (GANs)~\cite{goodfellow2014generative, gulrajani2017improved} as two examples of metrics, the latter proved to be very effective when analyzing high dimensional data~\cite{donoho2000high} in our problems. 

The deep generative model is physics-informed in that our partial physical knowledge of the dynamics will be encoded into the architecture or the loss function. Such physical knowledge is sometimes essential for a correct dynamic inference, since the particle dynamic can be not unique even if the curve $t\rightarrow \rho_t$ and its governing PDE are fully given. For example, the following two particle dynamics with $\mathcal{N}(0,1)$ as the initial distribution will lead to the same curve $\rho_t = \mathcal{N}(0,t+1)$:
\begin{itemize}
  \item Standard Brownian motion with no drift.
  \item $dx_t/dt = x_t/(2t+2)$ with no diffusion, i.e., $x_t = x_0\sqrt{t+1}$.
\end{itemize}
However, if we know that the particle dynamic is in the form of $d\bm{x}_t/dt = \bm{v}_t(\bm{x},t)$, so that $\rho_t$ is governed by the continuity equation $d\rho_t/dt = -\nabla \cdot (\bm{v}_t \rho_t)$, and $\bm{v}_t$ is limited to the $L^2(\rho_t; \mathbb{R}^d)$ closure of $\{\nabla\varphi:\varphi \in C_c^{\infty}(\mathbb{R}^d)\}$, then the solution of $\bm{v}_t$ is unique, for any curve $t\rightarrow \rho_t$ absolutely continuous from $[a,b]$ to $\mathcal{P}_2(\mathbb{R}^d)$, where $\mathcal{P}_2(\mathbb{R}^d)$ is the Wasserstein-2 space of probability measures with finite quadratic moments in $\mathbb{R}^d$ \cite{ambrosio2008gradient, ambrosio2008hamiltonian}.


In the generative model, we use discretized ODE or SODE with unknown terms parameterized as neural networks. This is referred as the ``neural ODE'' and ``neural SDE'' in the literature \cite{chen2018neural, li2020scalable, jia2019neural, liu2019neural, tzen2019neural, tzen2019theoretical}, but in applications the observations are mainly a time series, i.e., observations of a single particle. The idea of employing a neural networks as velocity surrogates with physics-informed loss functions in particle systems was also used for solving mean field game/control problems~\cite{ruthotto2020machine}.
There are other works using GANs to solve inverse problems, including \cite{yang2020physics}, where GANs were applied to learn parameters in time-independent stochastic differential equations, and \cite{liu2020rode} where GANs were applied to learn the random parameters from the (paired) observations of independent ODEs. 

In this paper, we tackle two typical types of particle systems: (1) independent particle systems governed by SODEs, and (2) interacting particle systems governed by nonlocal flocking dynamics. For inferring dynamics governed by SODE, most algorithms are based on observations of a sample path, and perform the inference by calculating or approximating the probability of observations conditioned on the system parameters, using Euler-Maruyama discretization~\cite{elerian2001likelihood, eraker2001mcmc}, Kalman filtering~\cite{sarkka2015posterior}, variational Gaussian process-based approximation~\cite{archambeau2007gaussian, vrettas2015variational}, etc. There are other works inferring the stochastic dynamics with the (estimated) densities from the perspective of Fokker-Planck equations~\cite{bakarji2021data}, but this approach is hard to scale to high dimensional problems. For flocking dynamic systems, other researchers infer the key parameters in the influence function by fitting the velocity field via Bayesian optimization~\cite{mao2019nonlocal}, or fitting the density field by solving a system of transformed PDEs and optimizing the parameters~\cite{mavridis2020learning}. But these methods require the knowledge of the initial condition, including the distribution and velocity field.

\section*{Problem Setup}\label{sec:problem}

\begin{figure*}[ht]
    \centering
    \includegraphics[width = 0.9\textwidth]{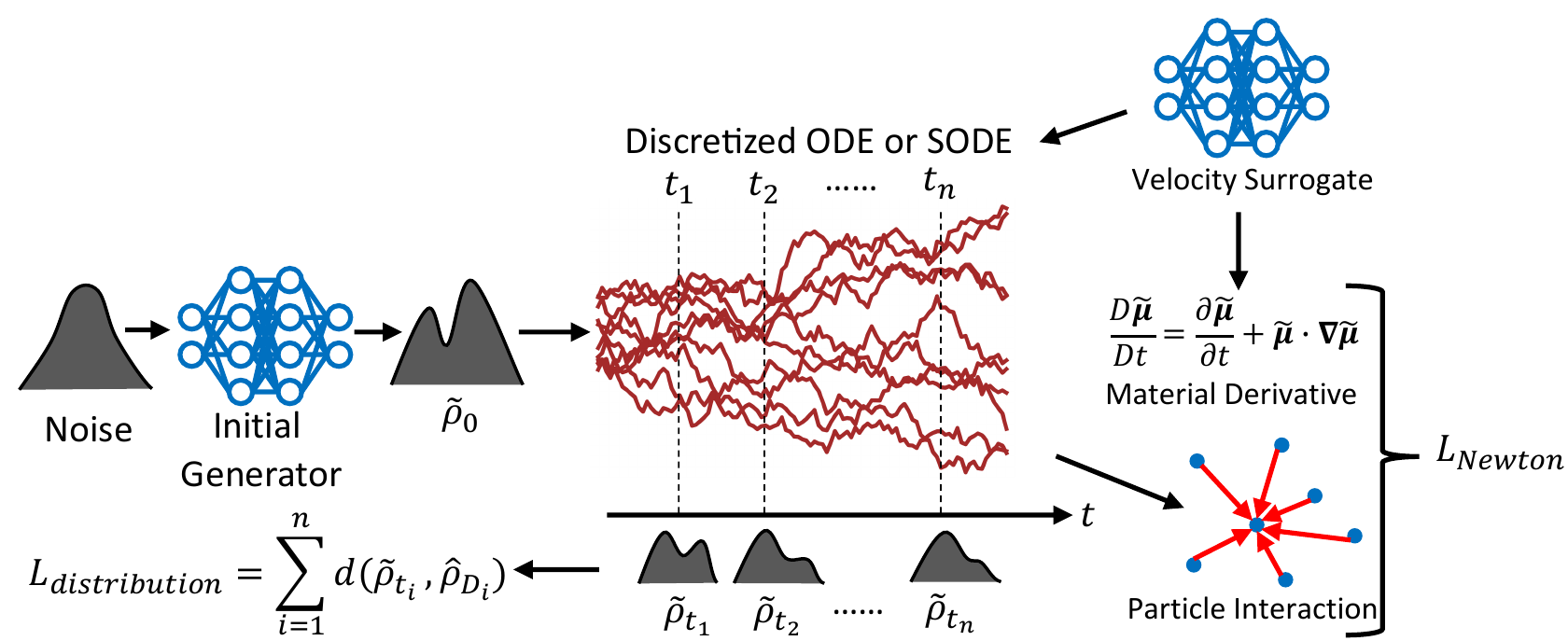}
    \caption{Schematic of the generative model for ensemble-regression. We first use a feed-forward neural network to map the input noise to the output $\tilde{\bm{X}}_0 \in \mathbb{R}^d$, whose distribution $\tilde{\rho}_0$ is intended to approximate the initial distribution $\rho_0$. Subsequently, we apply the discretized ODE or SODE with trainable parameters to generate particle trajectories $\tilde{\bm{X}}_t$ for $t>0$ with $\tilde{\bm{X}}_0$ as the initial condition (brown curves). Differences between the distributions of $\tilde{\bm{X}}_t$ and the snapshots from data are quantified as (a part of) our loss function. For interacting particle systems, a neural network is employed as the velocity surrogate to generate the particle trajectories. Another loss function term, namely the Newton loss $L_{\text{Newton}}$, is defined to quantify the consistency between the particle accelerations derived from the material derivatives, and the forces calculated from the particle interactions. With $L_{\text{Newton}}$, we enforce the inferred velocity to be consistent with our partial knowledge of the dynamics.}
    \label{fig:schematic1}
\end{figure*}

We start from a system consisting of an ensemble of particles, where the dynamics of the particles is independent of other particles. The most commonly used stochastic processes in physics and biology are diffusion processes and L\'evy processes; in particular, the particle dynamics is governed by the stochastic differential equation:
\begin{equation}\label{eqn:sode_diff}
    \begin{aligned}
    d\bm{X}_t = \bm{\mu}_t dt &+ \boldsymbol{\sigma}_t d\bm{B}_t, \quad t \ge 0, 
    \end{aligned}
\end{equation}
for diffusion processes, and 
\begin{equation}\label{eqn:sode_levy}
    \begin{aligned}
    d\bm{X}_t = \bm{\mu}_t dt &+  \boldsymbol{\sigma}_t d\bm{L}_t^\alpha, \quad t \ge 0, 
    \end{aligned}
\end{equation}
for L\'evy processes, where $\bm{X}_t \in \mathbb{R}^d$ is the position of a particle at time $t$ with $\bm{X}_0$ randomly drawn from the initial distribution $\rho_0$, $\bm{\mu}_t \in \mathbb{R}^d$ is the deterministic drift, $\bm{B}_t$ and $\bm{L}_t^\alpha$ are the $d$-dimensional standard Brownian motion and the $\alpha$-stable symmetric L\'evy process, respectively, and $\boldsymbol{\sigma}_t \in \mathbb{R}^{d\times d}$ is the diffusion coefficient. In the mean field limit, the density of the particles would be governed by Fokker-Planck equations or fractional Fokker-Planck equations. For simplicity, in this work we assume that $\bm{\mu}_t$ is a function of $\bm{X}_t$ while $\bm{\sigma}_t$ is constant, but, in principle, our proposed method can also tackle the time-dependent case.

Apart from systems consisting of independent particles, we further consider systems where particles interact with each other and the particle distributions exhibit more complicated behavior. As an example, we consider the Cucker-Smale particle model~\cite{cucker2007emergent}, which describes individuals in flocks with nonlocal interactions. The individual or particle motion is characterized by the following governing equations:
\begin{equation}\label{eqn:ode_cs}
    \begin{aligned}
    d\bm{X}_t^{(i)}/dt &= \bm{\mu}_t^{(i)}, \\
    d\bm{\mu}_t^{(i)}/dt = \frac{1}{N-1}\sum_{j\ne i} &\phi(\Vert\bm{X}_t^{(i)} - \bm{X}_t^{(j)}\Vert)(\bm{\mu}_t^{(j)}-\bm{\mu}_t^{(i)}),
    \end{aligned}
\end{equation}
where the superscripts denote the index of the particles, $N\gg 1$ is the number of particles in the system, and $\phi$ is the influence function. Here, we set
\begin{equation}\label{eqn:influence}
    \begin{aligned}
   \phi(r) = c_{d,\alpha} |r|^{-(d+\alpha)}, c_{d,\alpha} = \frac{\alpha\Gamma(\frac{d+\alpha}{2})}{2\pi^{\alpha+d/2}\Gamma(1-\alpha/2)},
    \end{aligned}
\end{equation}
where $\Gamma$ is the gamma function, $d$ is the dimension of $\bm{X}_t^{(i)}$, and $\alpha$ is the factor that characterizes the decay rate of particle interactions as the distance $r$ grows. In the mean field limit, as $N$ grows the density of the particles is governed by the following fractional PDE:
\begin{equation}\label{eqn:macro_cs}
    \begin{aligned}
    \frac{\partial \rho}{\partial t} + \nabla\cdot(\rho \bm{u})&=0, \\
    \frac{\partial\bm{u}}{\partial t} + \bm{u} \cdot \nabla \bm{u} &= [\mathcal{L}, \bm{u}](\rho), \\
     [\mathcal{L}, \bm{u}](\rho)(\bm{x}) = \text{p.v. } c_{d, \alpha} \int_{\mathbb{R}^d} &\frac{\bm{u}(\bm{y}) -\bm{u}(\bm{x})}{\Vert \bm{x}-\bm{y} \Vert^{d+\alpha}}\rho(\bm{y})d\bm{y},
    \end{aligned}
\end{equation}
where $\rho$ is the density, $\bm{u}$ is the velocity field, and p.v. means the principle value. Note that $\alpha$ determines the fractional order.

We consider the scenario where the data available are the observations of the particles coordinates at different time instants $\{t_i\}_{i=1}^n$ with $0 \le t_1 < t_2 ... <t_n$, namely ``snapshots''. In other words, the data $\mathcal{D}_i$ for time $t_i$ will be a set of samples drawn from $\rho_{t_i}$, the particle distributions at $t_i$. 
If $\{\mathcal{D}_i\}_{i=1}^n$ are observations of the same set of particles and we can distinguish these particles, we refer to these cases as the ``paired'' observations since we can pair the particles from different snapshots. In other cases, $\{D_i\}_{i=1}^n$ are observations of different sets of particles, or we cannot distinguish the particles. We refer to these cases as the ``unpaired'' observations. In this paper we mainly focus on the unpaired cases, but we also present some work on the paired cases in Paired Observations section, where we introduce how to reduce the effective dimensionality and we prove a theorem to support the introduced method. 
%
For independent particle systems, we assume that we are unaware of $\bm{\mu}_t$ and $\boldsymbol{\sigma}_t$, or we may only know the parametrized forms of these terms, and we aim to infer these terms directly or through a proper parametization. For interacting particle systems, we assume we know the forms of the dynamics and the influence function, i.e., Equation~\ref{eqn:ode_cs} and \ref{eqn:influence}, but need to infer the velocity field and the key parameter $\alpha$.



\section*{Physics-informed Generative Model}\label{sec:Methodology}

To perform ensemble regression, we will use a generative model with deep neural networks to represent a curve in the probability measure space. Note that the curve is determined by the initial condition $\rho_0$ and the governing equation, hence, the generative model  consists of two parts. In the first part, we employ a feed-forward neural network $G$ to represent $\rho_0$. In particular, $G$ takes samples from random noise $\mathcal{N}$, e.g., Gaussian noise, as input and the generated distribution ${G}_{\#}\mathcal{N}$ is intended to approximate $\rho_0$,  where $_{\#}$ denotes the push forward operator. The second part of the generative model will generate ``fake'' particle trajectories with initial coordinates generated by $G$, and the marginal distribution $\tilde{\rho}_t$ at time $t \ge 0$ will be used to represent $\rho_t$ in the curve.

Our knowledge of the physics will be incorporated into the generative model in two ways. For non-interacting particle systems, our knowledge including the form of the drift and diffusion as well as the type of stochastic processes will be directly embedded into the architecture of the generative model in the second part. For interacting particle systems, while we have no direct knowledge of the velocity field, we will enforce the inferred velocity to be consistent with our knowledge of the dynamics with a physics-based soft penalty. In the following, we introduce the details of the learning algorithm for both types of systems separately. A schematic overview of the method is shown in Figure~\ref{fig:schematic1}.

\subsection*{Non-interacting Particle Systems}

For non-interacting particle systems, generating particle trajectories is relatively straightforward by directly applying the discretization of governing SODE or ODE. For example, if the particle trajectories are diffusion processes or L\'evy processes, we can use the following forward Euler scheme:
\begin{equation}\label{eqn:num_diff}
    \begin{aligned}
    \tilde{\bm{X}}_0 &= G(\bm{z}), \quad \bm{z}\sim \mathcal{N} \\
    \tilde{\bm{X}}_{(i+1)\Delta t} &= \tilde{\bm{X}}_{i\Delta t} + \bm{\mu}_t \Delta t + \bm{\sigma}_t \sqrt{\Delta t}\bm{\xi}_i, \quad i\ge 0,\\
  \text{or }  \tilde{\bm{X}}_{(i+1)\Delta t} &= \tilde{\bm{X}}_{i\Delta t} + \bm{\mu}_t \Delta t + \bm{\sigma}_t \Delta t^{1/\alpha}\bm{\zeta}_{\alpha,i}, \quad i\ge 0,
    \end{aligned}
\end{equation}
where $\Delta t$ is the time step, $\bm{\xi}_i$ and $\bm{\zeta}_{\alpha,i}$ are i.i.d. standard Gaussian random variables and $\alpha$-stable random variables, respectively. We could represent $\bm{\mu}_t$ and $\bm{\sigma}_t$ with neural networks if they are unknown, or represent the unknown parameters with trainable variables if we know their parameterized form.

Our target is to tune the trainable variables in the generative model, including the parameters in $G$ and those for parameterizing $\bm{\mu}_t$ and $\bm{\sigma}_t$, so that the generated marginal distribution $\tilde{\rho}_{t_i}$ fits the data $\mathcal{D}_i$ for each $i$. We thus need to define a distance function $\mathsf{d}(\cdot, \cdot)$ to measure the difference between the two input distributions, which can be estimated from samples drawn from the two distributions. Consequently the loss function in non-interacting particle systems is defined as:
\begin{equation}\label{eqn:loss_for}
    \begin{aligned}
    L_{\text{distribution}} &= \sum_{i=1}^{n} \mathsf{d}(\tilde{\rho}_{t_i}, \hat{\rho}_{\mathcal{D}_i}), 
    \end{aligned}
\end{equation}
where $\hat{\rho}_{\mathcal{D}_i }$ is the empirical distribution induced from the sample set $\mathcal{D}_i$. We will refer to it as the distribution loss.

There could be many ways to define $\mathsf{d}$, including Wasserstein distances, maximum mean discrepancy, etc. In this paper we use two approaches to define $\mathsf{d}$.

First we choose the squared sliced Wasserstein-2 (SW) distance~\cite{santambrogio2017euclidean, deshpande2018generative} as the function $\mathsf{d}$:  \begin{equation}
    \mathsf{d}(\mu, \nu) =  SW_2^2(\mu, \nu) := \int_{\mathbb{S}^{d-1}} W_2^2({\pi_{\bm{e}}}_{\#}\mu,{\pi_{\bm{e}}}_{\#}\nu )d\mathcal{H}^{d-1}({\bm{e}}),
\end{equation}
where $W_2$ is the Wasserstein-2 distance, and ${\pi_{\bm{e}}}_{\#}\mu$ is the one dimensional distribution induced by projecting $\mu$ onto the direction $\bm{e}$, defined by
\begin{equation}
    ({\pi_{\bm{e}}}_{\#}\mu)(A) = \mu(\{\bm{x} \in \mathbb{R}^d: {\bm{e}}\cdot \bm{x} \in A\}), \forall A \in \mathcal{B}(\mathbb{R}),
\end{equation}
similarly for ${\pi_{\bm{e}}}_{\#}\nu$. $\mathcal{H}^{d-1}$ is the uniform Hausdorff measure on the sphere $\mathbb{S}^{d-1}$. In short, the squared sliced Wasserstein-2 distance is the expectation of the squared Wasserstein-2 distance between the two input measures projected onto uniformly random directions. The sliced Wasserstein-2 distance is exactly the Wasserstein-2 distance for one-dimensional distributions, but is easier to calculate for higher dimensional distributions. We present the details of the estimation in the Supplementary Information section S1.
    
We also use GANs to obtain $\mathsf{d}$. The generative model we introduced above can generate ``fake'' samples $\tilde{\bm{X}}_{t_i}$ at $\{t_i\}_{i=1}^n$, and for each $i$, we use a discriminator $D_i$ to discriminate generated samples $\tilde{\bm{X}}_{t_i}$ and real samples $\bm{X}_{t_i}$ from $\mathcal{D}_i$. The adversarial loss given by $D_i$ can act as a metric of the difference between $\tilde{\rho}_{t_i}$ and $\hat{\rho}_{\mathcal{D}_i}$. In particular, we use WGAN-GP~\cite{gulrajani2017improved} as our version of GANs in our paper, with 
    \begin{equation} \label{eqn:lossforWGANGP_G}
    \begin{aligned}
        \mathsf{d}(\tilde{\rho}_{t_i}, \hat{\rho}_{\mathcal{D}_i}) &= -\mathbb{E}_{\tilde{\bm{X}}_{t_i}\sim \tilde{\rho}_{t_i}}[D_i(\tilde{\bm{X}}_{t_i})] + \mathbb{E}_{\bm{X}_{t_i}\sim \hat{\rho}_{\mathcal{D}_i}}[D_i(\bm{X}_{t_i})],
    \end{aligned}
    \end{equation}
and the loss function for each discriminator $D_i$ is defined as
    \begin{equation} \label{eqn:lossforWGANGP_D}
    \begin{aligned}
	    L_{D_i} =& \mathbb{E}_{\tilde{\bm{X}}_{t_i}\sim \tilde{\rho}_{t_i}}[D_i(\tilde{\bm{X}}_{t_i})] - \mathbb{E}_{\bm{X}_{t_i}\sim \hat{\rho}_{\mathcal{D}_i}}[D_i(\bm{X}_{t_i})] \\
	    & + \lambda \mathbb{E}_{\hat{\bm{x}}_i\sim \rho_{\hat{\bm{x}}_i} } [(\Vert\nabla_{\hat{\bm{x}}_i} D_i(\hat{\bm{x}}_i) \Vert_2 - 1)^2], \text{  for } i = 1,2...n,
    \end{aligned}
    \end{equation}
    where $\rho_{\hat{\bm{x}}_i}$ is the distribution generated by uniform sampling on straight lines between pairs of points sampled from $\tilde{\rho}_{t_i}$ and $\hat{\rho}_{\mathcal{D}_i}$, and $\lambda = 0.1$ is the gradient penalty coefficient. 
    Here, $\mathsf{d}(\tilde{\rho}_{t_i}, \hat{\rho}_{\mathcal{D}_i})$ can be mathematically interpreted as the Wasserstein-1 distance between the two input distributions. WGAN-GP is computationally more expensive than the sliced Wasserstein distance since we need to train the generative model and the discriminators iteratively, but it is more scalable to high dimensional problems, for which we made a comparison in the supplementary information section S1.


\begin{figure*}[ht]
\centering
    \begin{subfigure}{1.0\textwidth}
    \centering
    \includegraphics[width =0.19 \textwidth]{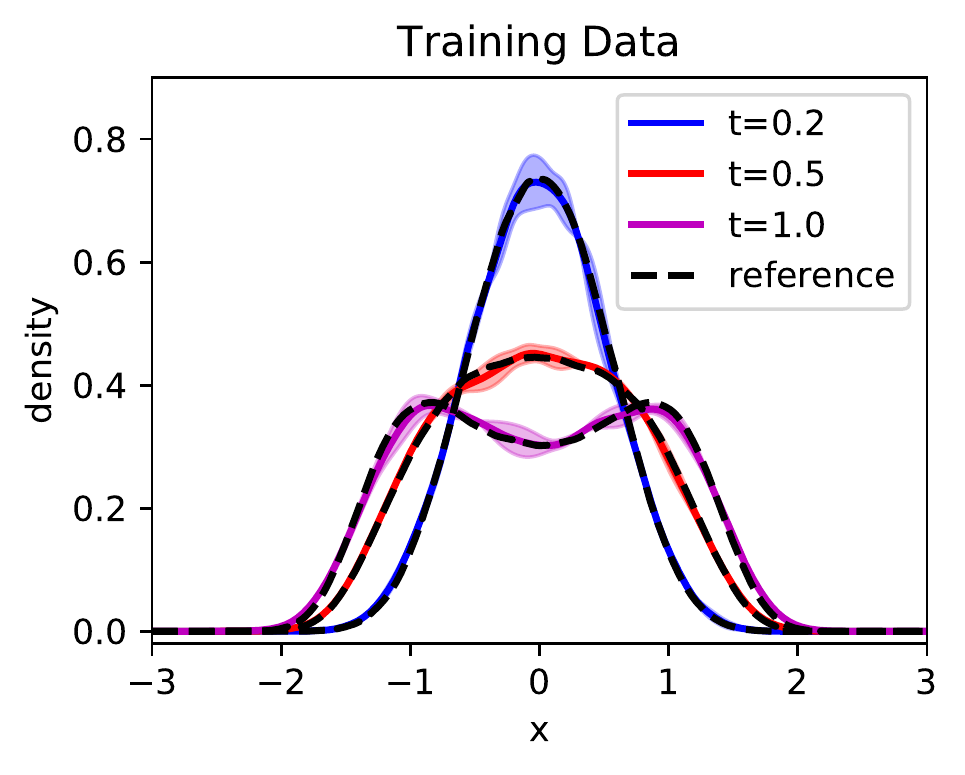}
    \includegraphics[width =0.19 \textwidth]{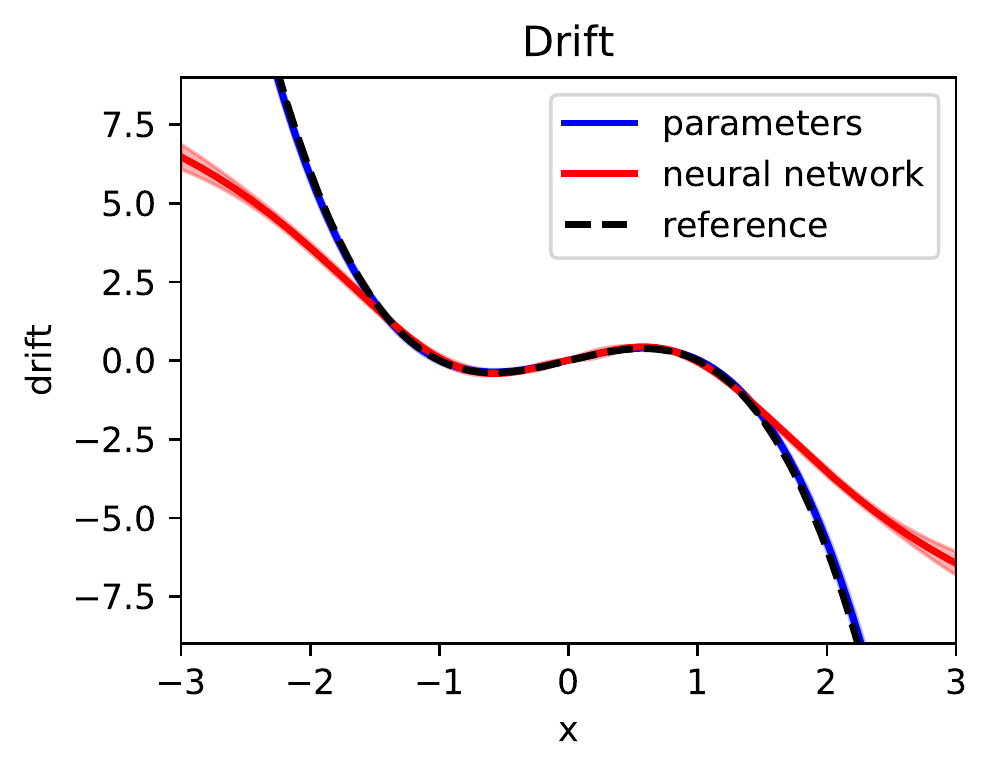}
    \includegraphics[width =0.19 \textwidth]{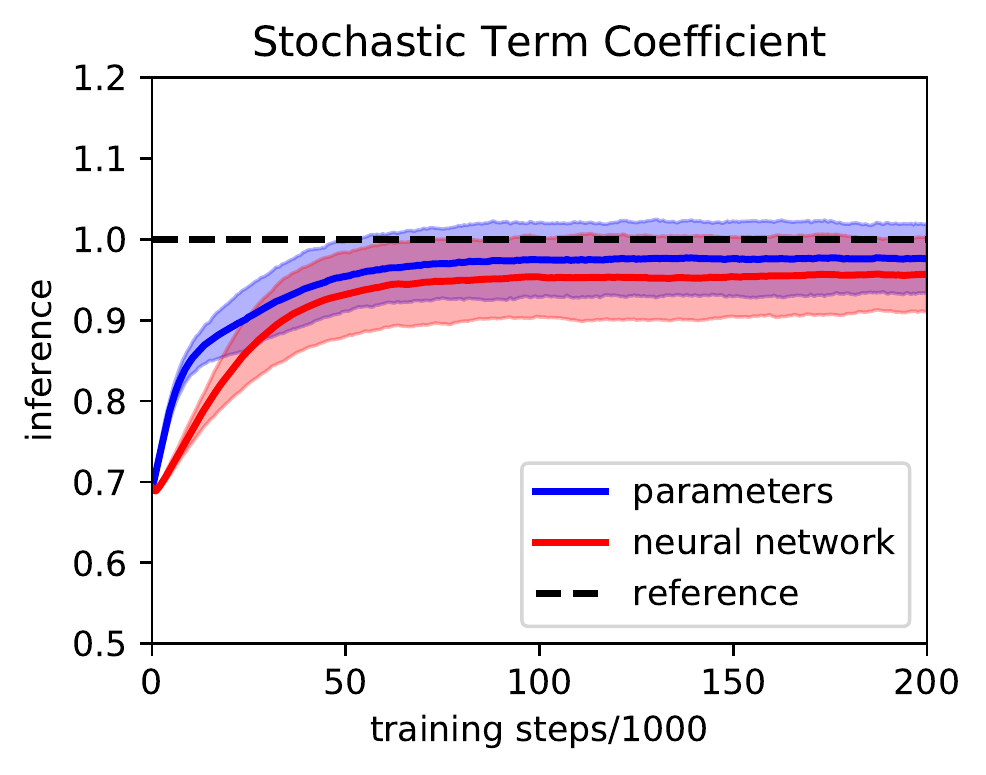}
    \includegraphics[width =0.19 \textwidth]{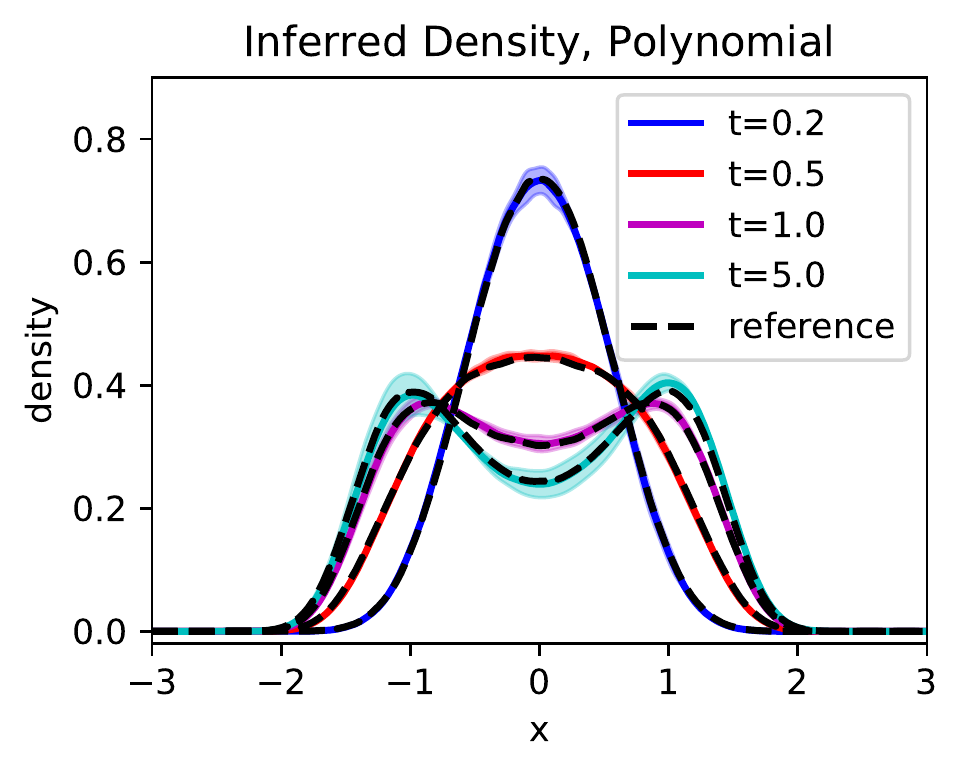}
    \includegraphics[width =0.19 \textwidth]{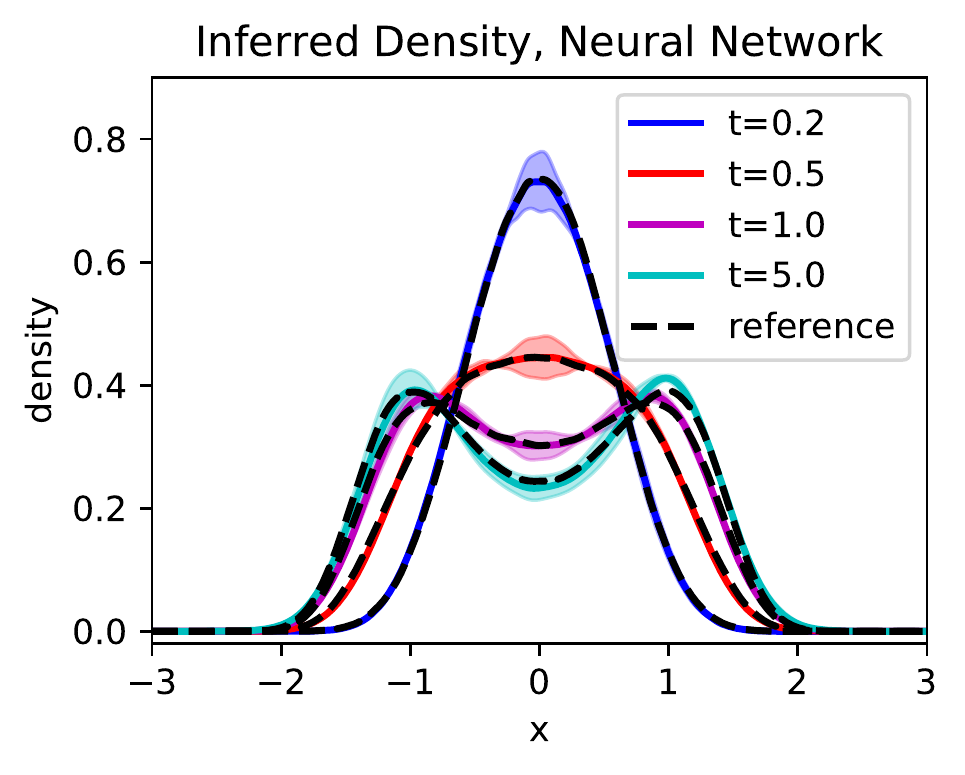}
    \caption{}
    \label{fig:Inverse1DBrown}
    \end{subfigure}
    \begin{subfigure}{1.0\textwidth}
    \centering
    \begin{subfigure}{0.2\textwidth}
    \centering
    \includegraphics[width = \textwidth]{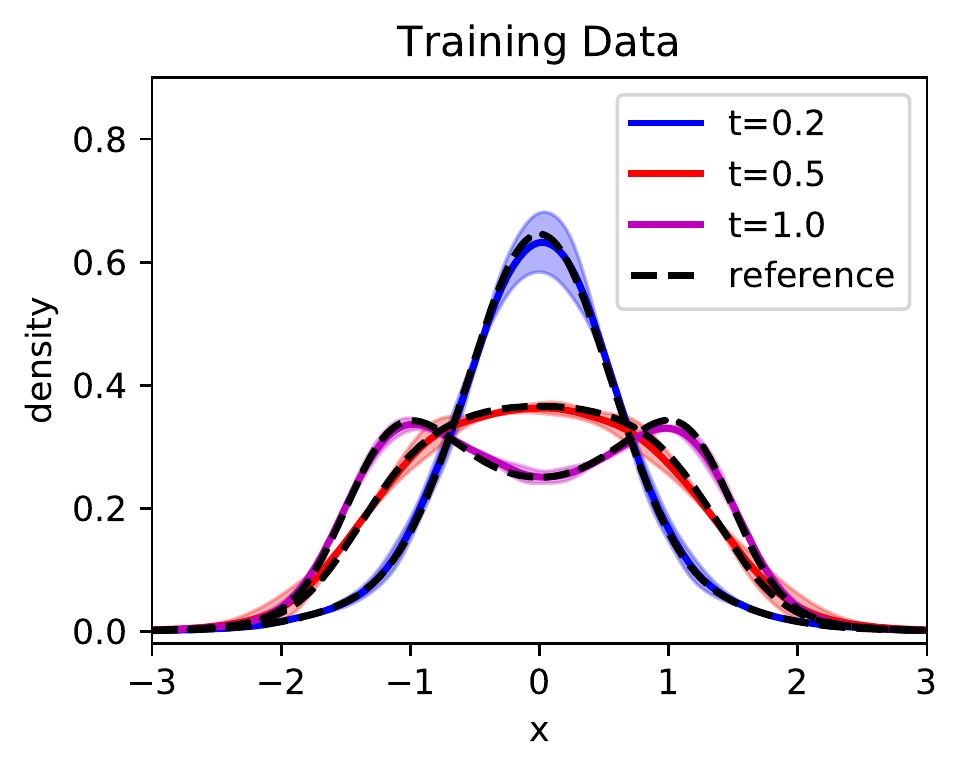}
    \end{subfigure}
    \begin{subfigure}{0.79\textwidth}
    \centering
    \includegraphics[width =0.24 \textwidth]{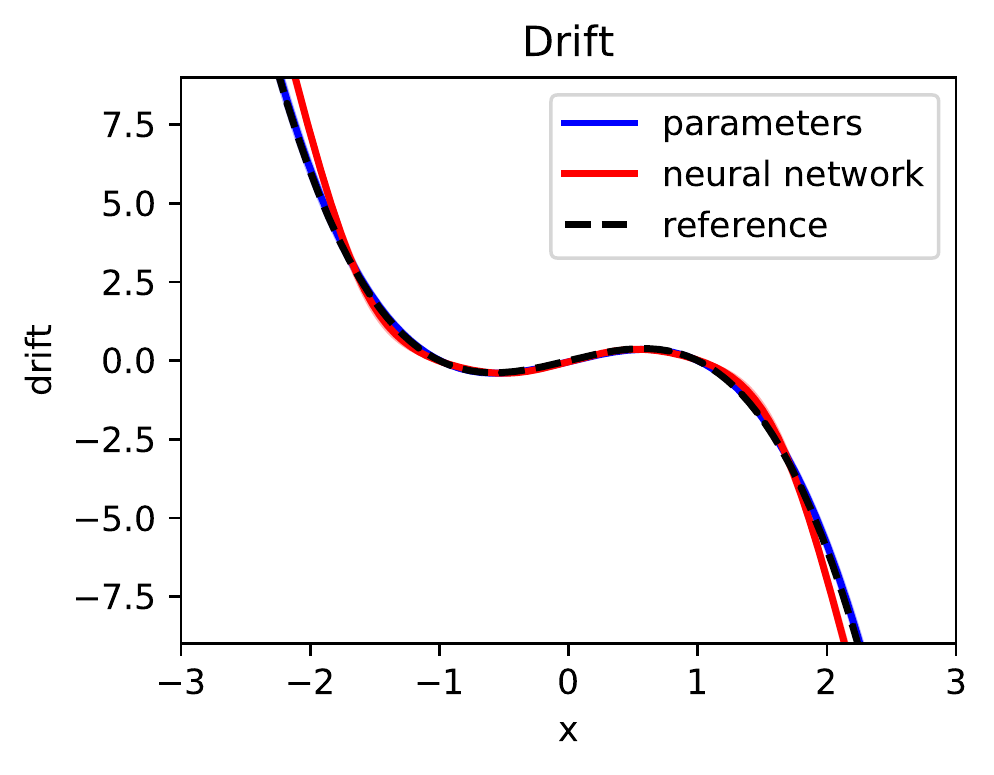}
    \includegraphics[width =0.24 \textwidth]{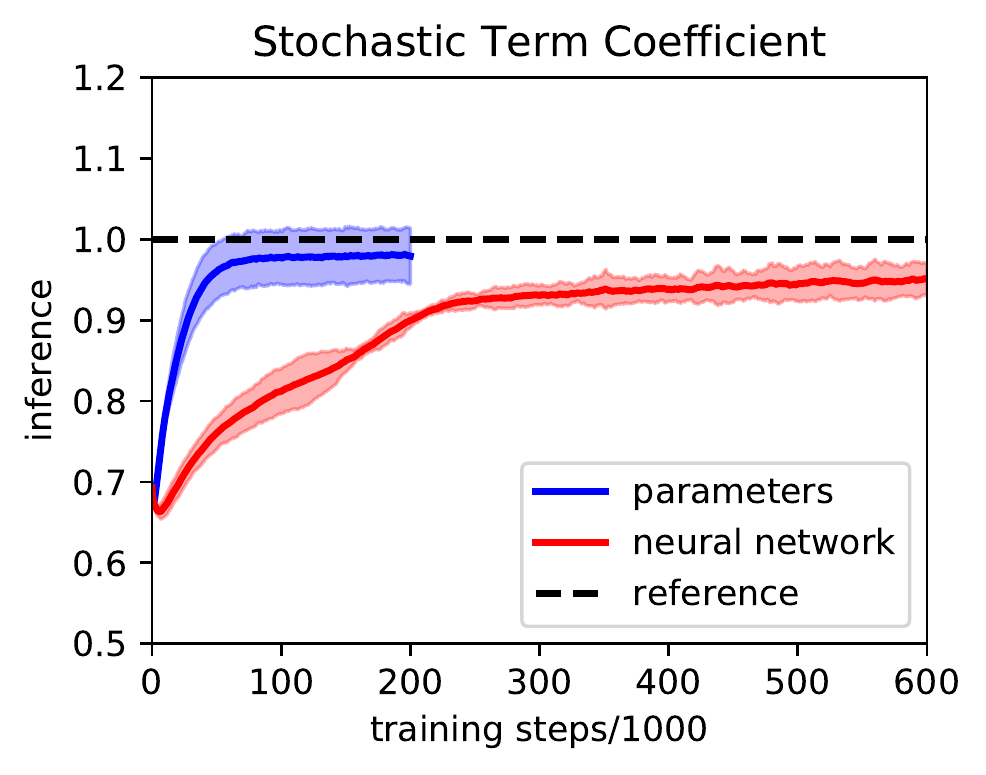}
    \includegraphics[width =0.24 \textwidth]{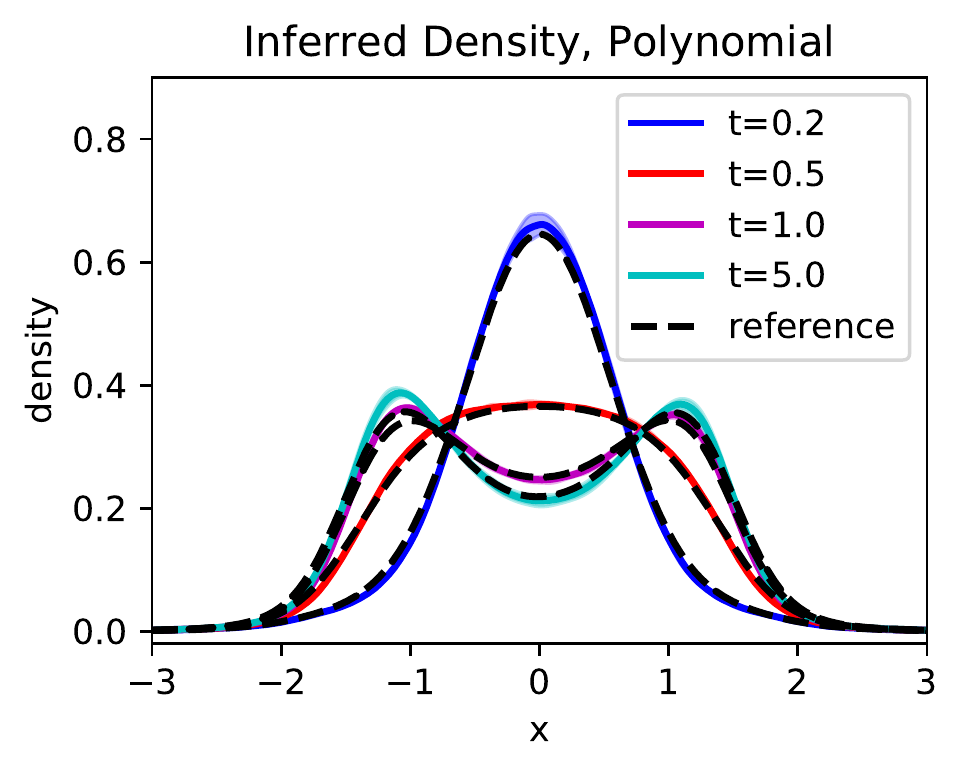}
    \includegraphics[width =0.24 \textwidth]{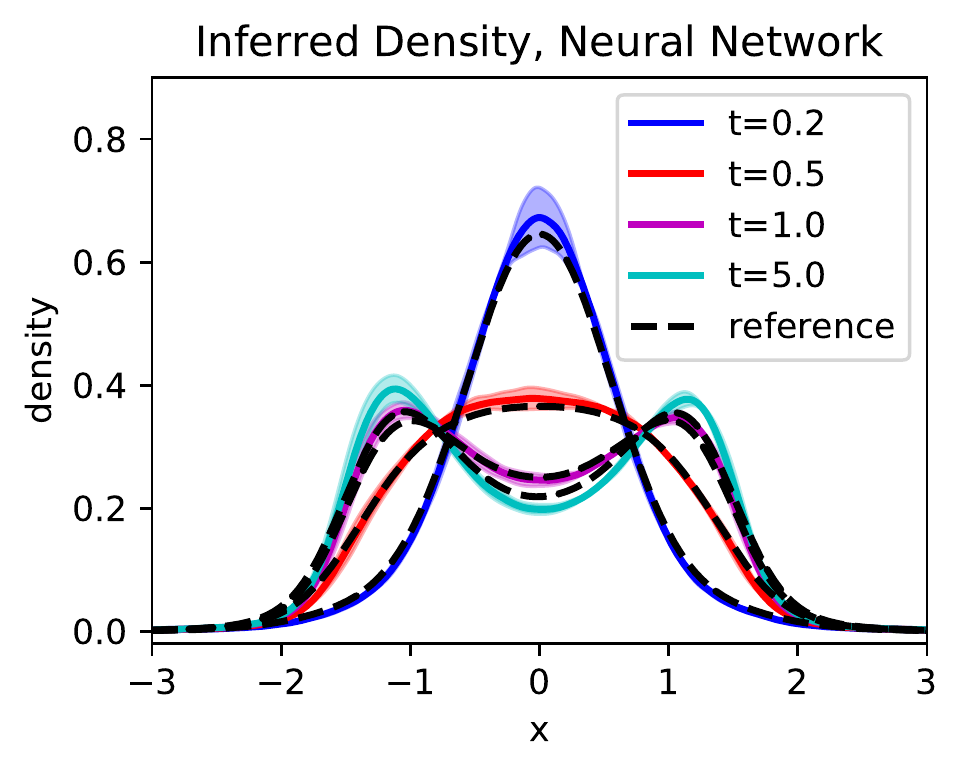}
    
    \includegraphics[width =0.24
    \textwidth]{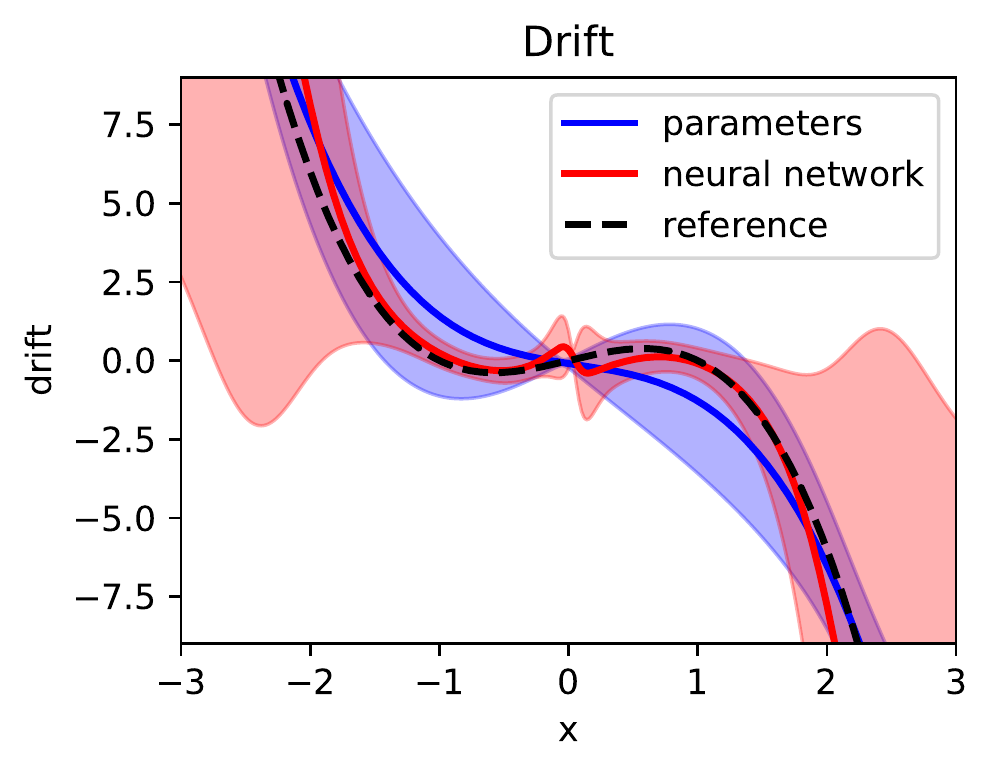}
    \includegraphics[width =0.24 \textwidth]{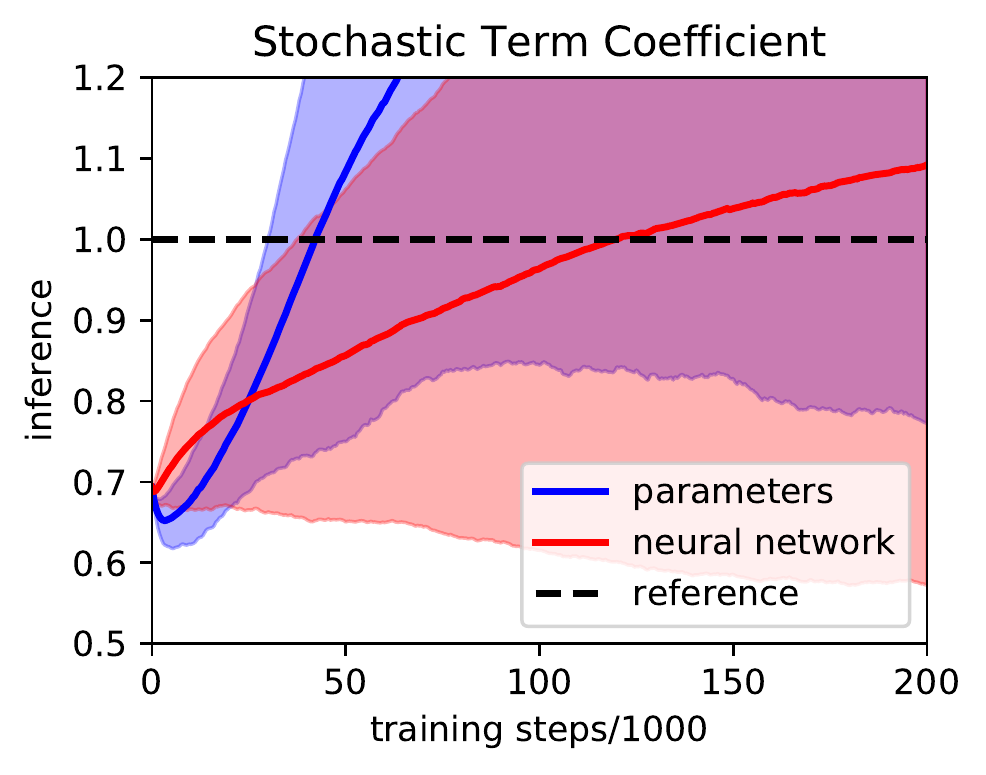}
    \includegraphics[width =0.24 \textwidth]{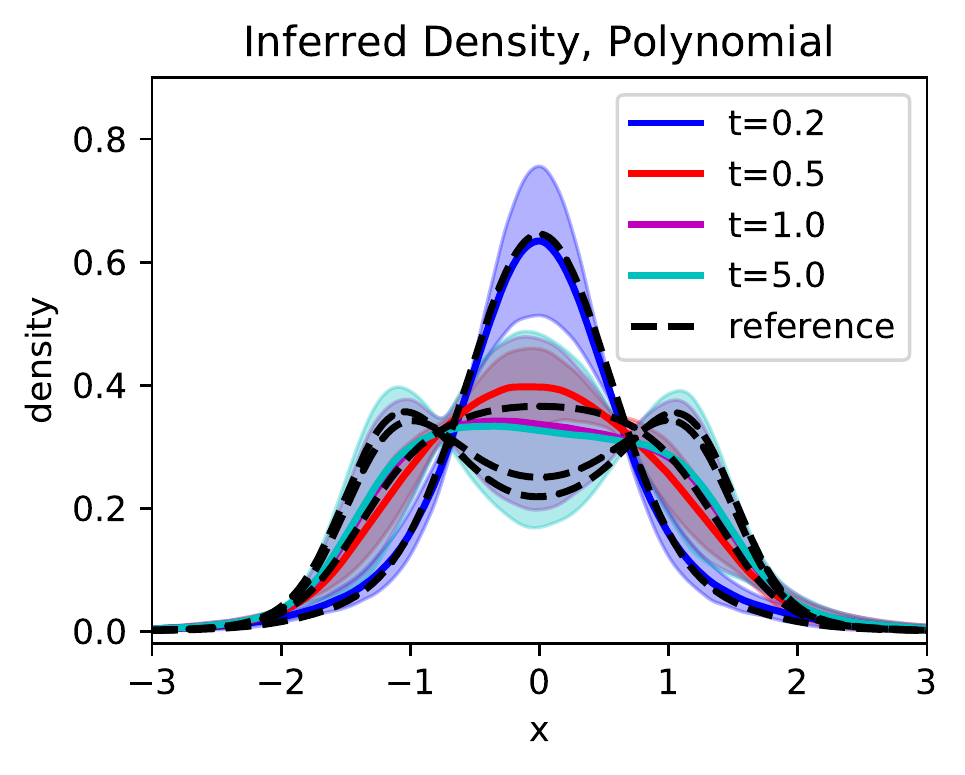}
    \includegraphics[width =0.24 \textwidth]{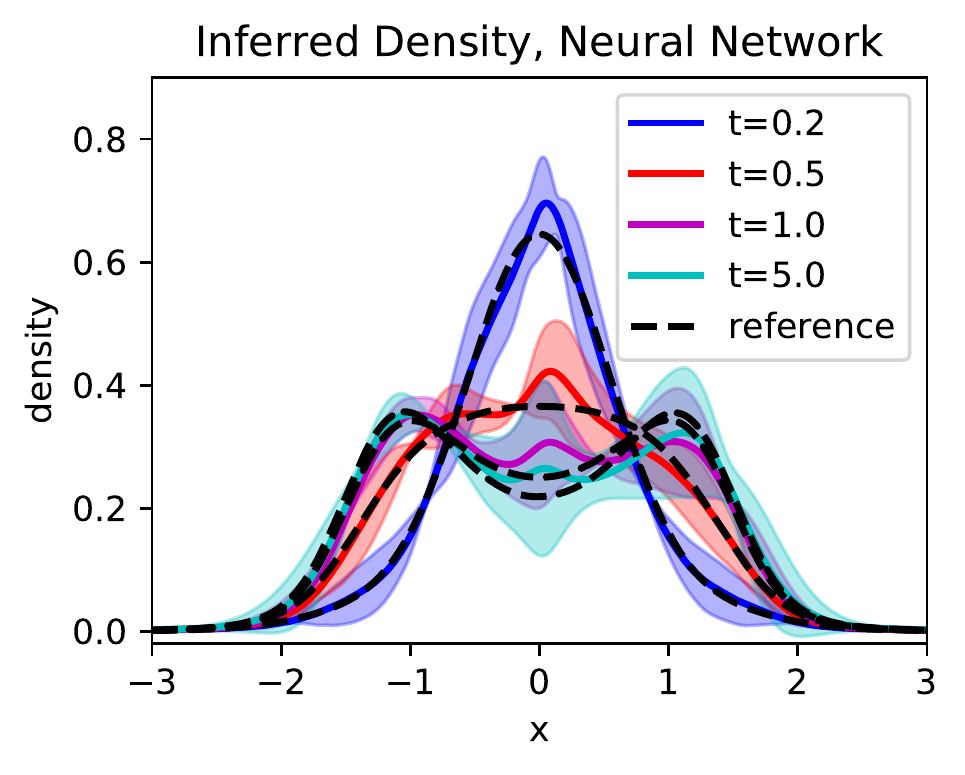}
    \end{subfigure}
    \caption{}
    \label{fig:Inverse1DLevy}
    \end{subfigure}
    
    \caption{Results for 1D problems with (a) Brownian Motion and (b) L\'evy process as the stochastic term. The two rows of (b) show the results with (top row) and without (bottom row) bounded map pre-processing, respectively. The first column visualizes densities estimated from training data. The second column shows the inferred drift functions in the end of training. The third column shows the inferred coefficient for the stochastic term during training. The fourth and fifth columns show the inferred densities in the two cases of using a cubic polynomial and a neural network to parameterize the drift, respectively. The solid lines and shaded areas refer to the mean and two standard deviations of three runs with different data. }
    \label{fig:Inverse1D}
\end{figure*}

\subsection*{Interacting Particle Systems}

The straightforward discretization of the governing Equations~\ref{eqn:ode_cs} cannot be directly applied to generate ``fake'' particle trajectories in interacting particle systems, especially for those with strong nonlocal interactions, since the computational cost for one time step would be $O(N^2)$, where $N$ is the number of particles required so that the system is close to the mean field limit, which makes the learning almost intractable. Instead, we propose to first employ a neural network $\tilde{\bm{\mu}}$ as a surrogate model of the velocity field $\bm{\mu_t}$ in the spatial-temporal domain to generate trajectories, and then apply an additional penalty in the loss function to enforce the velocity field $\tilde{\bm{\mu}}$ to be consistent with the Equations~\ref{eqn:ode_cs}. In this paper we use the forward Euler scheme to generate trajectories:
\begin{equation}\label{eqn:num_flock}
    \begin{aligned}
    \tilde{\bm{X}}_0 &= G(\bm{z}), \quad \bm{z}\sim \mathcal{N} \\
    \tilde{\bm{X}}_{(i+1)\Delta t} &= \tilde{\bm{X}}_{i\Delta t} + \tilde{\bm{\mu}}(\tilde{\bm{X}}_{i\Delta t}, i\Delta t) \Delta t, \quad i\ge 0,
    \end{aligned}
\end{equation}
where $\Delta t$ is the time step. By employing the surrogate model for velocity, the computational cost for generating the particle trajectories is linear, instead of quadratic, with the number of particles.

To enforce $\tilde{\bm{\mu}}$ to be consistent with our knowledge of the dynamics of Equations~\ref{eqn:ode_cs}, we first randomly generate $K$ particle trajectories with Equation~\ref{eqn:num_flock}, denoted as $\{\{\tilde{\bm{X}}^{(k)}_{i\Delta t}\}_{i=0}^{I}\}_{k=1}^K$, and then calculate the ``forces'' applied on these $k$ particles by another $M$ randomly generated particles $\{\{\tilde{\bm{Y}}^{(m)}_{i\Delta t}\}_{i=0}^{I}\}_{m=1}^M$, using the formula of interactions in Equations~\ref{eqn:ode_cs}, where we replace the unknown parameter $\alpha$ with a trainable variable. Meanwhile, from the velocity field $\tilde{\bm{\mu}}$ we can directly calculate the accelerations for these $k$ particles using a material derivative. The forces and accelerations should be consistent with each other. We can thus define an $L_2$ loss: 
\begin{equation}\label{eqn:newton}
    \begin{aligned}
    &L_{\text{Newton}} = \frac{1}{I+1}\sum_{i=0}^I\frac{1}{K}\sum_{k=1}^K(\bm{F}_i^{(k)}- \bm{a}_i^{(k)})^2,\\
    \bm{F}_i^{(k)} =& \frac{1}{M}\sum_{m=1}^M \phi(\Vert\tilde{\bm{X}}^{(k)}_{i\Delta t} - \tilde{\bm{Y}}^{(m)}_{i\Delta t}\Vert)(\tilde{\bm{\mu}}(\tilde{\bm{Y}}^{(m)}_{i\Delta t}, i\Delta t)- \tilde{\bm{\mu}} (\tilde{\bm{X}}^{(k)}_{i\Delta t}, i\Delta t)),\\
    \bm{a}_i^{(k)} =& \frac{D\tilde{\bm{\mu}}(\tilde{\bm{X}}^{(k)}_{i\Delta t}, i\Delta t))}{Dt} \\=&  \frac{\partial \tilde{\bm{\mu}}(\tilde{\bm{X}}^{(k)}_{i\Delta t}, i\Delta t))}{\partial t} + \tilde{\bm{\mu}}(\tilde{\bm{X}}^{(k)}_{i\Delta t}, i\Delta t)) \cdot \nabla \tilde{\bm{\mu}}(\tilde{\bm{X}}^{(k)}_{i\Delta t}, i\Delta t)),
    \end{aligned}
\end{equation}
and we name it as the Newton loss. The time span  $[0,I\Delta t]$ should cover the time for the latest snapshot $t_n$, and in this paper we simply set $I$ as the ceiling of $t_n / \Delta t$. To reduce computational cost, the average over $I+1$ time steps in $L_{\text{Newton}}$ can also be approximated by mini-batch, i.e., taking average over $B$ random time steps in each training step.

For one time step, the computational cost is $O(KM)$ for calculating the force terms and $O(K)$ for the acceleration terms. Note that $M$ should be larger than or equal to $N$, but $K$ can be much less than $N$, thus the total computational cost for $L_{\text{Newton}}$ can be much less than $O(N^2)$, and this makes the learning tractable.

In the end, the loss for interacting particle systems will be a combination of the Newton loss and the distribution loss:
\begin{equation}\label{eqn:loss_flock}
    \begin{aligned}
    L = \eta L_{\text{Newton}} + \frac{1}{n}  \sum_{i=1}^{n} \mathsf{d}(\tilde{\rho}_{t_i}, \hat{\rho}_{\mathcal{D}_i}),
    \end{aligned}
\end{equation}
where the weight $\eta$ is set as 1 in this paper.

\subsection*{Modifications to the Distributions}\label{sec:modication}

In order to provide flexibility and adapt the method to different problems, we can also modify the generated distribution and real data before feeding them to the distance function $\textsf{d}$ and calculate the distribution loss. We present some examples here.

In some cases $\rho_t$ may have heavy tails, e.g., when the particle trajectories correspond to a L\'evy process. The heavy tails could spoil the training, since the rare outlier samples could dominate the loss function. We can choose a suitable bounded map $h: \mathbb{R}^d \rightarrow \mathbb{R}^d$ to preprocess both the generated samples and the real data so that the heavy tails are removed.
If the observations of the particle coordinates are noisy, we can also perturb the generated samples to add artificial noise. By scaling the artificial noise with trainable variables, the size of the noise in observations can also be learned.
If we can only make observations of particles in a specific domain, i.e., the observations are truncated, we will also filter the generated samples using a corresponding mask so that the effective domain for the generated samples and observations are the same.

\section*{Computational Results}\label{sec:results}

All the neural networks in the main text are feed-forward neural networks with three hidden layers, each of width 128, except the discriminator for high dimensional problems which is a ResNet with 5 hidden layers, each of width 256, and shortcut connections skipping one hidden layer. We use the leaky ReLu \cite{maas2013rectifier} activation function with $\alpha=0.2$ for the discriminator neural networks in WGAN-GP, while using the hyperbolic tangent activation function for other neural networks. The neural network weights are initialized with the uniform Xavier initializer, and the biases are initialized as 0. The variables for parameterizing the diffusion and noise size are initialized as 0 (before been activated by the softplus function). The drift parameters are initialized as $0$ in 1D problems, $-0.5$ in high dimensional problems, and randomly initialized with standard Gaussian distribution in 2D problems. The batch size for the distribution loss is 1000 for non-interacting particle systems, except in the 2D problem with truncated observations, where we generate 10000 samples to compensate the loss of samples due to filtering. We use the Adam optimizer~\cite{kingma2014adam} with $lr =0.0001,\beta_1 = 0.9, \beta_2 = 0.999$ for the cases using the sliced Wasserstein distance, while $lr =0.0001,\beta_1 = 0.5, \beta_2 = 0.9$ for the cases with WGAN-GP. The time step in the generative model is set as $\Delta t = 0.01$.

\subsection*{1D Problems: Brownian Motion and L\'evy Process}

In this section, we test our method on the 1D problems using the sliced Wasserstein distance as $\mathsf{d}$. We first consider the SODE with Brownian motion and then with the $\alpha$-stable symmetric L\'evy process as the stochastic term: 
\begin{equation}\label{eqn:1DInv}
    \begin{aligned}
    dX_t &= (X_t - {X_t}^3) dt + dB_t,\\ 
    \text{or} \quad dX_t &= (X_t - {X_t}^3) dt +  dL_{t}^{\alpha}, \quad t \ge 0, 
    \end{aligned}
\end{equation}
where $\alpha = 1.5$, with $\rho_0 =  \mathcal{N}(0,0.04)$.

Note that we have no knowledge of $\rho_0$ in learning. Also, we assume we know that the stochastic term has a constant coefficient but we need to infer it; here the ground truth is $1.0$. To represent the unknown coefficient, we use a trainable variable rectified by a softplus function $\text{softplus}(x) = \ln(1+ e^x)$, which ensures positivity. As for the drift $\mu(x) = x- x^3$, we consider the following two cases for both SODE problems. In case 1, we know that the drift $\mu(x)$ is a cubic polynomial of $x$. In this case, we use a cubic polynomial $a_0 + a_1x + a_2 x^2 + a_3 x^3$ to parameterize $\mu(x)$, and want to infer the four coefficients $a_0, a_1, a_2$ and $a_3$. In case 2, we only know that the drift $\mu(x)$ is a function of $x$ and hence we use a neural network to parameterize $\mu(x)$.

For the SODE problem with Brownian motion, we first prepare a pool of $10^5$ sample paths, then independently draw $10,000$ samples at $t = 0.2, 0.5, 1.0$ from the pool as our training data. The results for both cases of drift parameterization are illustrated in Figure~\ref{fig:Inverse1DBrown}. In the figure, all the densities are estimated using Gaussian kernel density estimation with Scott's rule, and the inferred densities and the reference densities come from $10^6$ samples produced by the generative model or simulation. 

Both cases of drift parameterization provide a good inference of the diffusion coefficient, with an error less than $7\%$ after $2\times 10^5$ training steps, in all the runs. When using the cubic polynomial parameterization, the inferred drift fits well with the ground truth, with the relative $L_2$ error about $3\%$ in $[-3,3]$ averaged over three runs. The inferred drift using the neural network only fits the ground truth in the region between -1.5 and 1.5. This is reasonable since the particles mainly concentrate in this region, and we can hardly learn the drift outside this region, where the training data are sparse. However, we note that such an inference of drift is sufficiently good for an accurate time extrapolation of the particle distribution at $t = 5$.

One interesting observation is that in case 1 the inferred densities are more accurate than those estimated from the training data. This is because our knowledge of the governing SODE bridges the limited samples in three snapshots, and as we infer the density at, e.g., $t=0.2$, we are not only utilizing the data at $t=0.2$ but also implicitly leveraging the data at $t=0.5$ and $1.0$. We can make an analogy in the context of linear regression: multiple noisy data are helpful for inferring the hidden ground truth, since they are bridged by the regressed linear function. We present a more detailed discussion on this topic in the supplementary information in section S2.

For the SODE problem with L\'evy process, we prepare $1.5 \times 10^5$ sample paths, then independently draw $10,000$ samples within the region $[-1000, 1000]$ at $t = 0.2, 0.5, 1.0$ from the pool as our training data. To prevent instability during the training, we apply double precision and clip the generated $\alpha$-stable random variable in Equation \ref{eqn:sode_levy} between $-100$ and $100$.

We first test our method without pre-processing the samples as in the problem with Brownian motion. As we can see in the second row of Figure~\ref{fig:Inverse1DLevy}, the inferences are far away from the ground truth. This is due to the heavy tail of $\rho_t$ in the L\'evy process: some samples far away from $0$, although rare, could dominate the loss function and spoil the training. To deal with this problem, we then apply a bounded map $h(x)= 2 \tanh (x/2)$ to all the generated and real samples before feeding them to $\mathsf{d}$. The results are shown in the first row of Figure~\ref{fig:Inverse1DLevy} where we can see the inferences are much improved. In case 2 where the drift is parameterized by neural networks, the inferred drift outside $[-1.5,1.5]$ is better than that in the problem with Brownian motion. This is because the samples are more scattered in the L\'evy case.

\subsection*{2D Problems: Various Scenarios of Observations}\label{sec:2D}

In this section, we test our method on 2D problems using the sliced Wasserstein distance as $\mathsf{d}$, with various scenarios of observations. We consider the following 2D SODE:
\begin{equation}\label{eqn:2DMore}
    \begin{aligned}
    d\bm{X}_t &= \bm{\mu}(\bm{X_t})dt + 
    \begin{bmatrix}
    s_0 & 0\\
    s_1 & s_2
    \end{bmatrix} 
    d\bm{B}_t
    \end{aligned}
\end{equation}
where 
\begin{equation}\label{eqn:2DMore_drift}
\begin{aligned}
     \bm{\mu}(\bm{x}) = \nabla_{\bm{x}} \varphi(\bm{x}),
\end{aligned}
\end{equation}
and
\begin{equation}\label{eqn:2DMore_potential}
\begin{aligned}
    \varphi(\bm{x}) = & -(x_1 + a_0)^2(x_2 + a_1)^2\\
            &-(x_1 + a_2)^2(x_2 + a_3)^2 \quad \text{for } \bm{x} = (x_1, x_2). 
\end{aligned}
\end{equation}
The parameters are set as $a_0 = a_1 = s_0 = s_1 = s_2 = 1.0$, $a_2 = a_3 = -0.5$. 
We set the initial distribution as $\rho_0 =  \mathcal{N}(0,\bm{I}_2)$. We assume we know that the diffusion coefficient is a constant lower triangular matrix but we need to infer the three unknown parameters $s_0, s_1, s_2$. In particular, we use $(\text{softplus}(\tilde{s}_0), \tilde{s}_1, \text{softplus}(\tilde{s}_2))$ to approximate $(s_0, s_1, s_2)$, where $\tilde{s}_0, \tilde{s}_1, \tilde{s}_2$ are three trainable variables. Similar as in the 1D problems, we consider the following two cases for the drift. In case 1, we know the form of $\bm{\mu}$ and $\varphi$ in Equation~\ref{eqn:2DMore_drift} and \ref{eqn:2DMore_potential} but we need to infer $a_0, a_1, a_2$ and $a_3$ ($(a_0, a_1)$ and $(a_2, a_3)$ are exchangeable). In case 2, we only know that $\bm{\mu}$ is a gradient of $\varphi$ in Equation~\ref{eqn:2DMore_drift}, but we have no knowledge of $\varphi$. In this case, we use a neural network to parameterize $\varphi$. 


For the training data, we prepare $10^5$ sample paths, and consider the following various scenarios of observations at $t = 0, 0.1, 0.2, 0.3, 0.5, 0.7, 1.0$. The data are visualized in the supplementary information section S3.

\begin{itemize}
    \item Scenario 1: we assume that our observations are ideal, i.e., we make accurate observations of all the particle coordinates as the training data.
    \item Scenario 2: we assume that our observations are noisy. Specifically, we make observations of all the particle coordinates, but each coordinate is perturbed by an i.i.d. random noise $\mathcal{N}(\bm{0}, e^2 \bm{I}_2)$, where $e$ is set as $0.2$ but also need to be inferred during learning.
    \item Scenario 3: we assume that our observations are truncated. Specifically, we make observations of the particle coordinates in $\Omega = (\infty, 0.5) \times \mathbb{R}$, with the particles outside of $\Omega$ dropped.
\end{itemize}

For the first scenario, the inferred drift and diffusion are illustrated in Figure~\ref{fig:2DIdeal}. In case 1, all the inferred drift and diffusion parameters approach the ground truth during the training, with relative error less than $0.3\%$ for each $a_i$ and less than $4\%$ for each $\sigma_i$ after $2\times 10^5$ training steps. In case 2, the inferred diffusion parameters are also close to the ground truth, with error comparable with that in case 1. In the region where the training data are dense (i.e., the density estimated from all the training data is no less than 0.05), the inferred drift field has a relative $L_2$ error about $13\%$. The inference is much worse in other regions, and the reason is the same as in the 1D problem: the neural network can hardly learn the drift field where the training data are sparse.

For the second and third scenarios, we apply the technique of perturbing and filtering the generated samples, respectively. In Figure~\ref{fig:2DNoisy} and Figure~\ref{fig:2Dtrunc} we show the results of case 1 with the drift parameterized by four parameters. After $2\times 10^5$ training steps, all the parameters converge to ground truth. For scenario 2, the error is less than $0.5\%$ for each drift parameter $a_i$, less than $3\%$ for each diffusion parameter $s_i$, and about $5\%$ for the noise parameter $e$. For scenario 3, the error is less than $0.2\%$ for each drift parameter $a_i$, less than $1\%$ for each diffusion parameter $s_i$. We remark that we failed to learn the drift with neural network parameterization in the second and third scenario, suggesting a room for further improvements.

\begin{figure}[H]
    \centering
    \begin{subfigure}{0.8\textwidth}
    \centering
    \includegraphics[width = 0.48\textwidth]{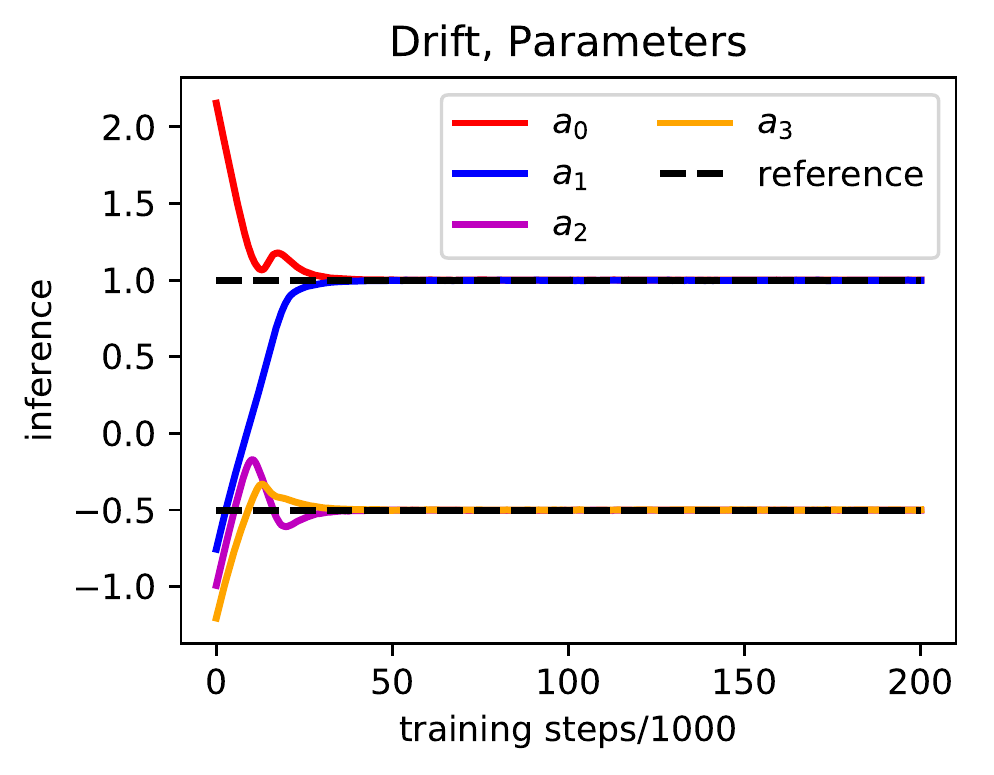}
    \includegraphics[width = 0.48\textwidth]{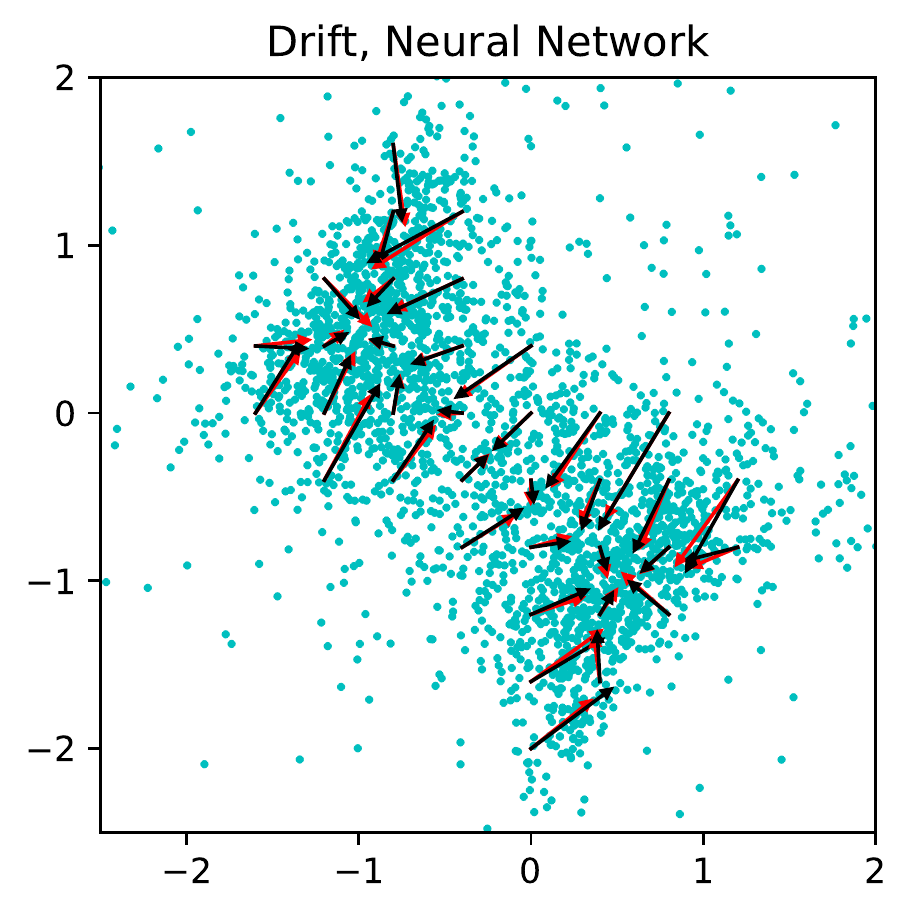}
    \includegraphics[width = 0.48\textwidth]{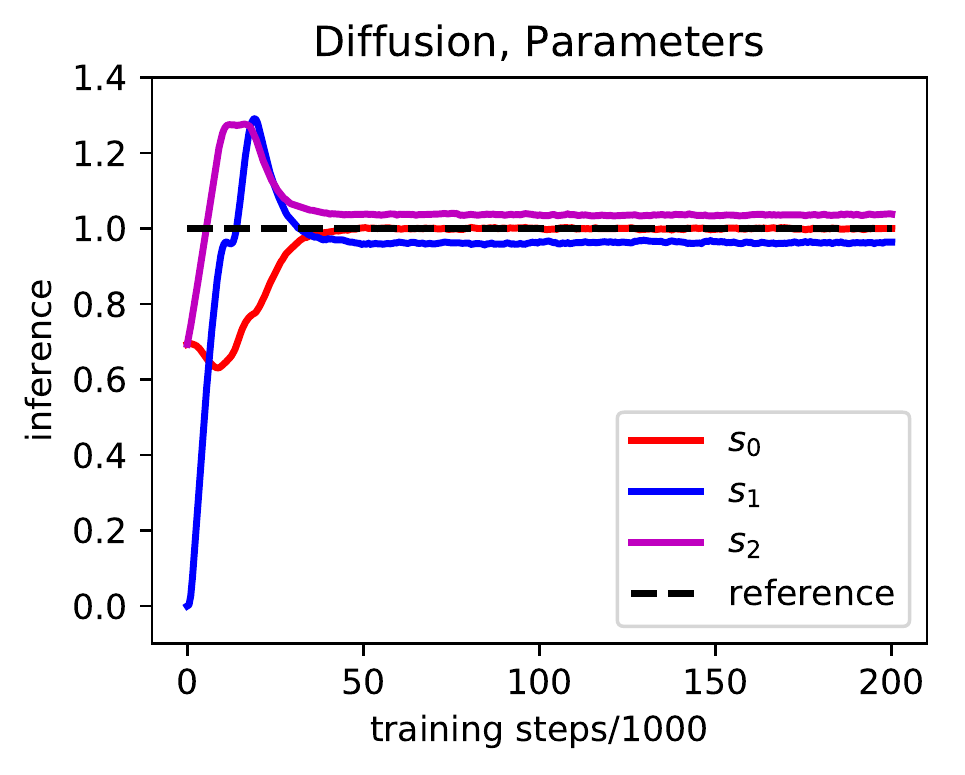}
    \includegraphics[width = 0.48\textwidth]{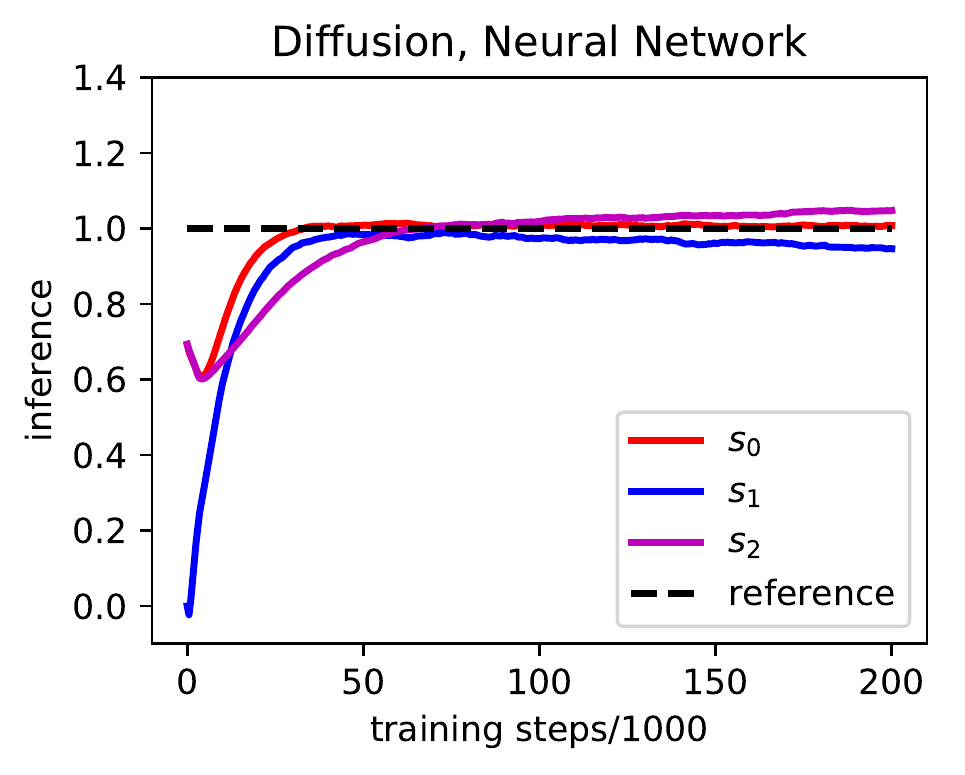}
    \caption{}
    \label{fig:2DIdeal}
    \end{subfigure}
    \begin{subfigure}{0.4\textwidth}
    \centering
    \includegraphics[width = \textwidth]{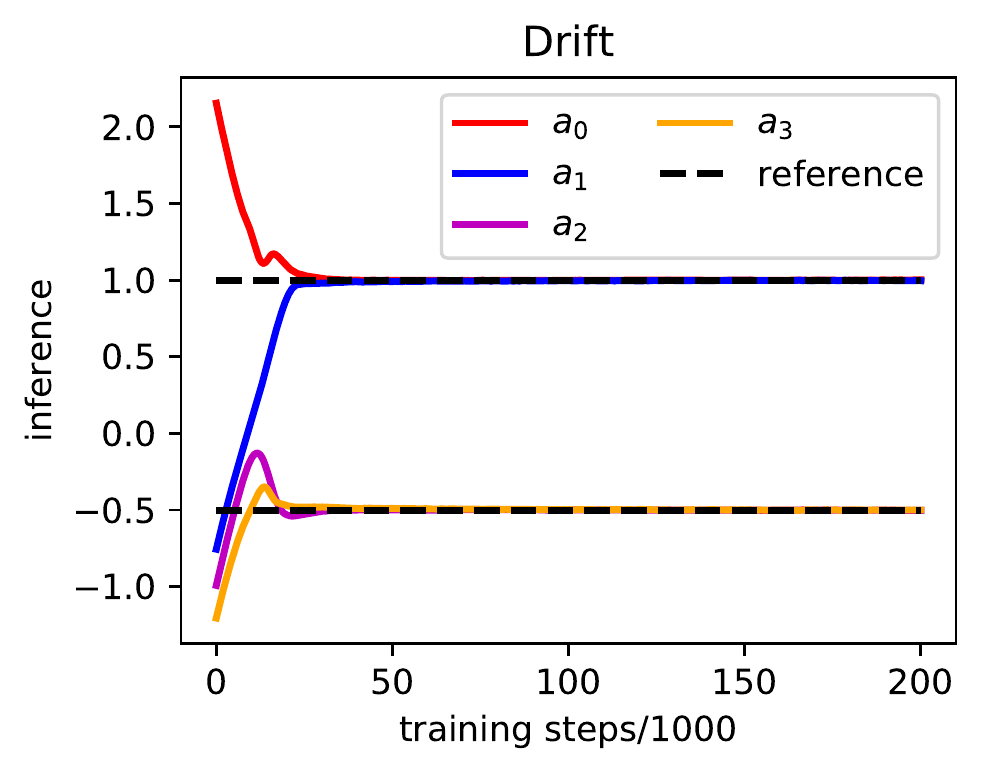}
    \includegraphics[width = \textwidth]{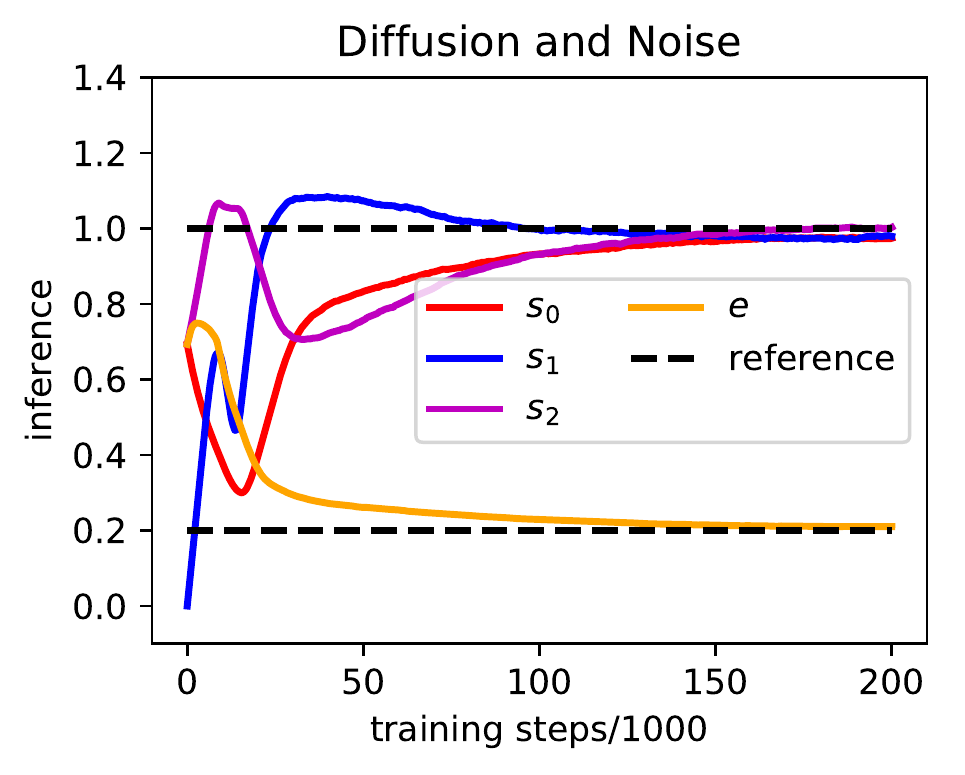}
    \caption{}
    \label{fig:2DNoisy}
    \end{subfigure}
    \begin{subfigure}{0.4\textwidth}
    \centering
    \includegraphics[width = \textwidth]{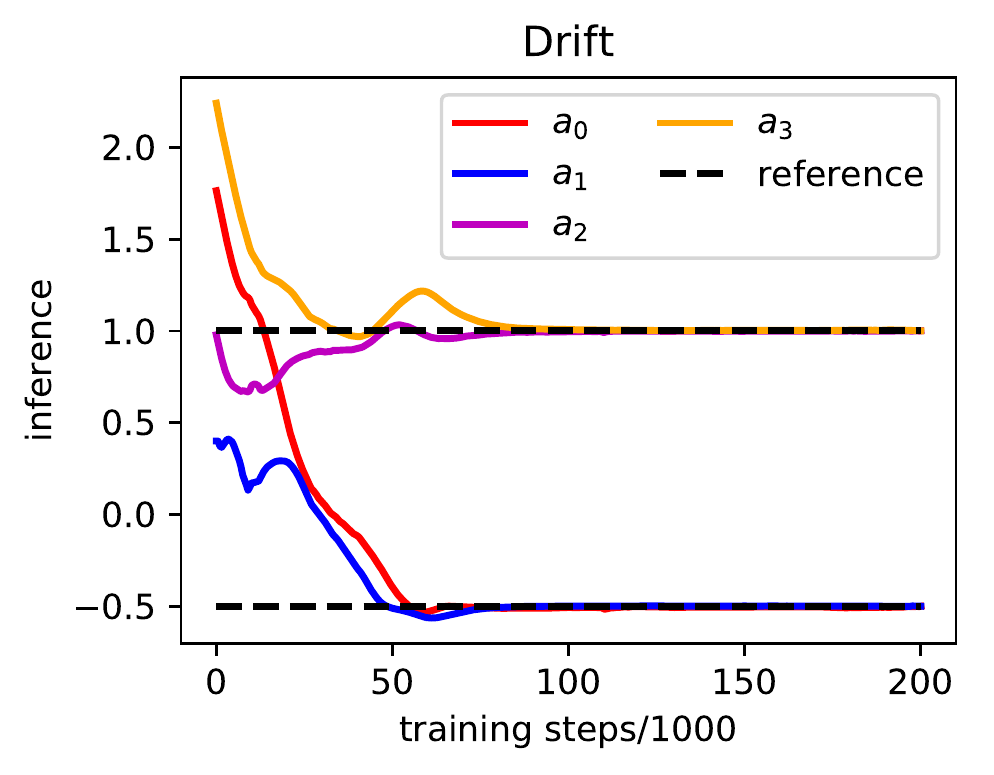}
    \includegraphics[width = \textwidth]{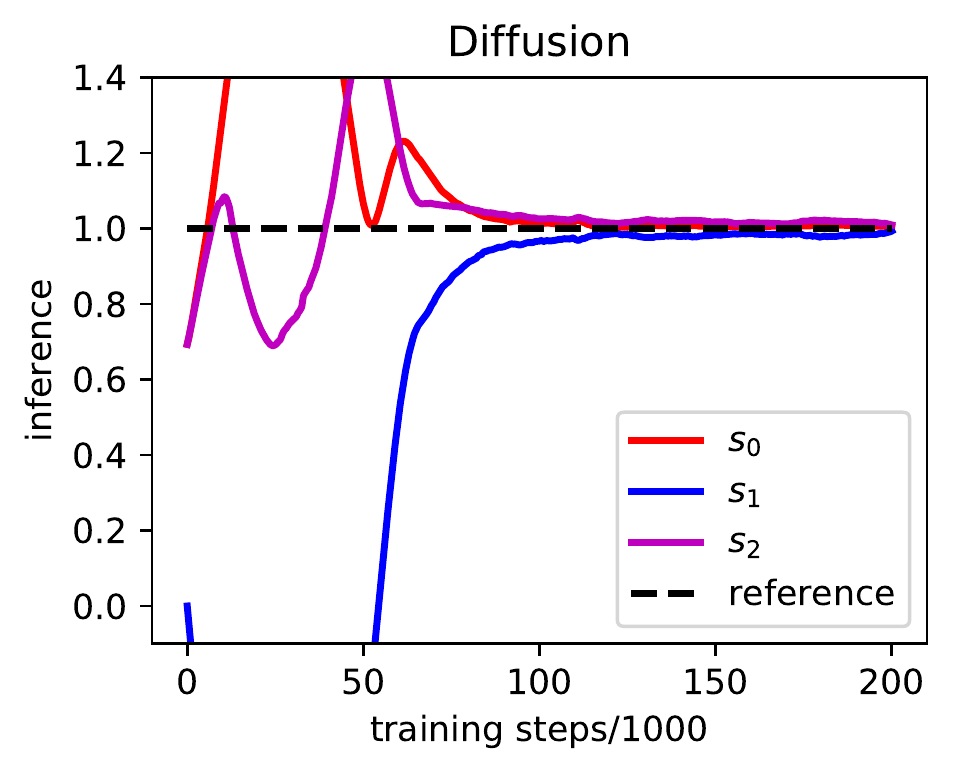}
    \caption{}
    \label{fig:2Dtrunc}
    \end{subfigure}
    
    \caption{Inferred drift and diffusion in 2D problems, with (a) ideal observations, (b) noisy observations, and (c) truncated observations. In (a), the left column shows the results in case 1 where the drift is parameterized by four parameters, while the right column shows the results in case 2 where the drift is parameterized by a neural network. In the top right figure of (a), we visualize the drift field where the training data are dense. The red and black arrows represent the inferred drift and the exact drift, respectively. For each arrow, the length represents the norm of the drift, scaled by $0.1$. The green dots are samples from the merged training data. }
    \label{fig:Inverse2DSW}
\end{figure}

\subsection*{Higher Dimensional Problems}\label{sec:highdim}
In this section, we test our method on higher dimensional non-interacting particle systems. Note that the sliced Wasserstein distance does not perform well in high dimensional problems, and we thus switch to the WGAN-GP to provide $\mathsf{d}$. The comparison between the sliced Wasserstein distance and WGAN-GP in high dimensional problems is presented in the supplementary information section S1.

We consider a $d$-dimensional SODE:
\begin{equation}\label{eqn:HighInvC}
    \begin{aligned}
    d\bm{X}_t &= \bm{\mu}(\bm{X_t})dt + \bm{\sigma}d\bm{B}_t
    \end{aligned}
\end{equation}
where 
\begin{equation}\label{eqn:HighInvC_drift}
\begin{aligned}
     \mu^{(i)}(\bm{X_t}) = X_t^{(i)} - (X_t^{(i)} )^3, i = 1,2...,d,
\end{aligned}
\end{equation}
for $\mu^{(i)}$ as the $i$-th component of $\bm{\mu} \in \mathbb{R}^d$, and $X_t^{(i)}$ is the $i$-th component of $\bm{X}_t \in \mathbb{R}^d$. We set the diffusion coefficient matrix as
\begin{equation}\label{eqn:HighInvC_diff}
    \begin{aligned}
    \bm{\sigma} =
    \begin{bmatrix}
    s_1  & 0   & 0 & 0 & \cdots & 0\\
    s'_2 & s_2 & 0 & 0 & \cdots & 0\\
    0 & s'_3 & s_3 & 0 & \cdots & 0\\
    0 & 0 & s'_4 & s_4 & \cdots & 0\\
    \vdots & \vdots & \vdots & \vdots & \ddots & \vdots\\
    0 & 0 & 0&  \cdots & s'_d& s_d\\
 \end{bmatrix} ,
    \end{aligned}
\end{equation}
where the $2d-1$ nonzero entries $\{s_i\}_{i=1}^d$ and $\{s'_i\}_{i=2}^d$ are set as 1. The initial distribution is $\rho_0 = \mathcal{N}(\bm{0}, 0.04 \bm{I}_d)$. Due to the non-diagonal diffusion coefficient matrix, the motion in different dimensions are coupled.

\begin{figure}[H]
    \centering
    \includegraphics[width = 0.35\textwidth]{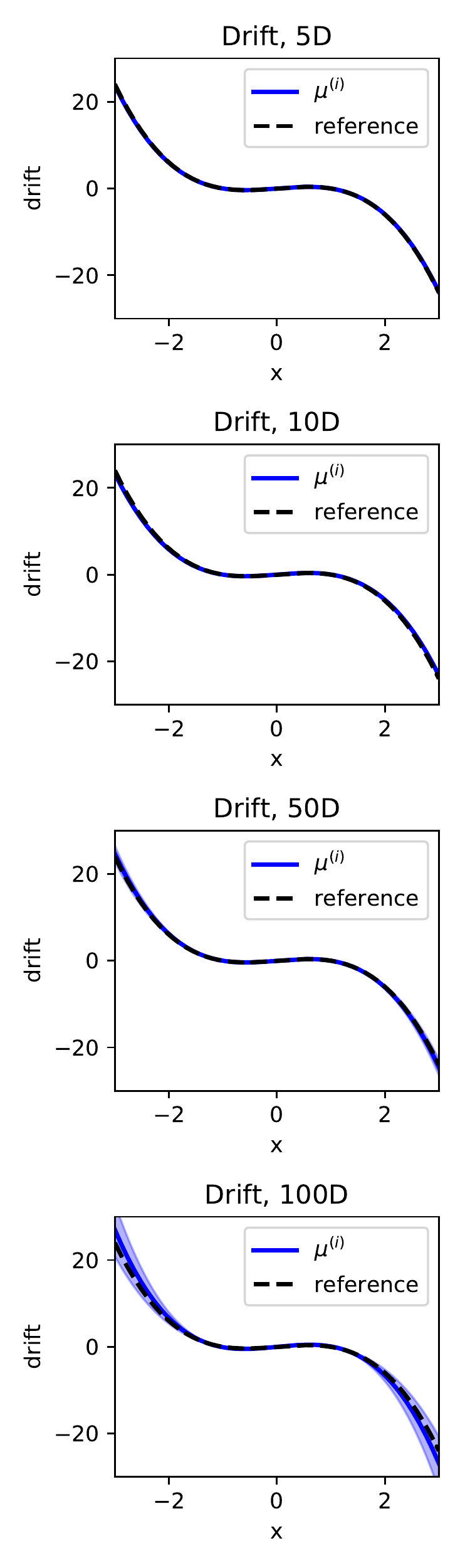}
    \includegraphics[width = 0.35\textwidth]{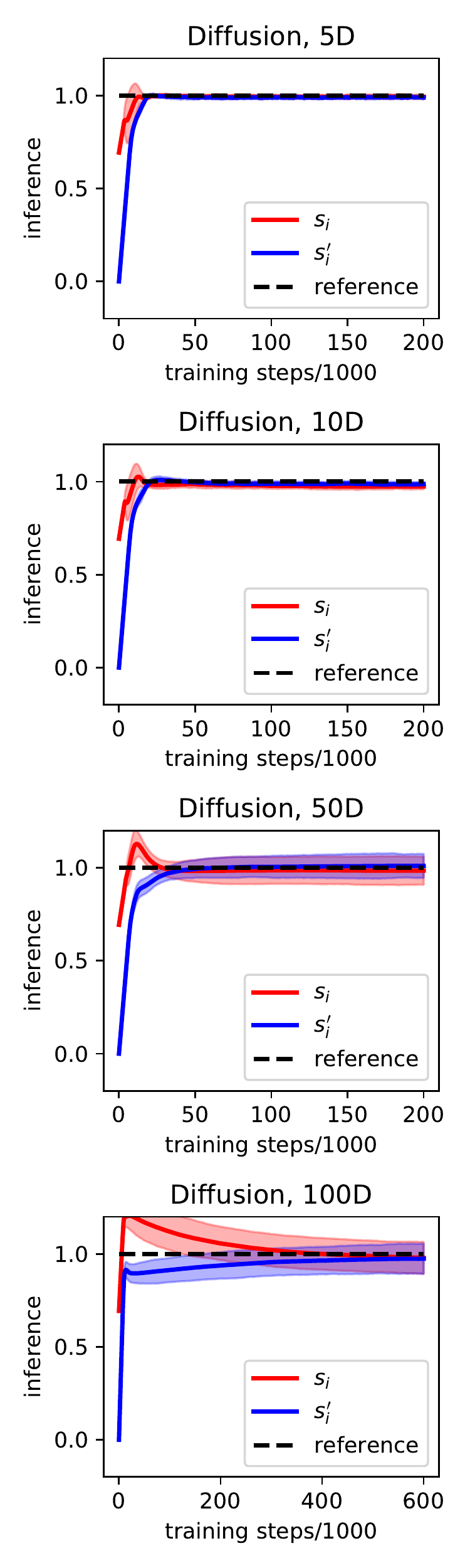}
    \caption{Inferred drift in the end of training and diffusion parameters during training, in 5D, 10D, 50D and 100D problems. The solid lines and shaded areas are the mean and two standard deviations over all $i$ and three runs with different random seeds.}
    \label{fig:Inverse20DC}
\end{figure}

We prepare $10^5$ sample paths and observe all the particle positions at $t = 0.2, 0.5, 1.0$ as our training data. We use a cubic polynomial with four variables to parameterize $\mu^{(i)}$ as a function of $X_t^{(i)}$ for each $i$, while using $2d-1$ trainable variables to learn the nonzero entries in the diffusion coefficient matrix, with the diagonal entries rectified by a softplus function for positivity. In other words, we use $6d-1$ variables to parameterize the drift and diffusion for the $d$-dimensional SODE. 
The results are shown in Figure~\ref{fig:Inverse20DC}. Even for the 100-dimensional problem, after $6\times 10^5$ training steps, the average error of the diffusion coefficients is about $0.04$, and the average relative $L_2$ error of the drift in the interval $[-3,3]$ in each dimension is about $14\%$. 

It may appear surprising that we can solve a 100-dimensional problem with only $10^5$ samples since usually an exponentially large number of samples is required to describe the distribution. However, we remark that the difficulty of the learning task is significantly reduced since our partial knowledge, i.e., the form of the SODE, is encoded into the generator.

\subsection*{Interacting Particle Systems}

In this section we consider the 1D and 2D cases of the interacting particle system with governing equations~\ref{eqn:ode_cs}. To prepare the training data, we follow \cite{mao2019nonlocal} and set the initial condition, including the density and velocity as
\begin{equation}
    \begin{aligned}
    \rho_0(x) &= \mathbbm{1}_{[-0.75,0.75]}\frac{\pi}{3}\cos\left(\frac{3\pi x}{2}\right), \\
    \mu_0(x) &= -\frac{1}{2}\sin\left(\frac{3\pi x}{2}\right),
    \end{aligned}
\end{equation}
for the 1D case and 
\begin{equation}
    \begin{aligned}
    \rho_0(x_1,x_2) &= \mathbbm{1}_{[-0.75,0.75]^2}\left(\frac{\pi}{3}\right)^2\cos\left(\frac{3\pi x_1}{2}\right)\cos\left(\frac{3\pi x_2}{2}\right),\\
    \bm{\mu}_0(x_1,x_2) &= \left(-\frac{1}{2\sqrt{2}}\sin(\frac{3\pi x_1}{2}),-\frac{1}{2\sqrt{2}}\sin(\frac{3\pi x_2}{2}) \right)
    \end{aligned}
\end{equation}
for the 2D case, where $\mathbbm{1}$ is the indicator function.  Using the Velocity-Verlet method with time step 0.01, we perform simulations with 1024 and 9976 particles for the 1D and 2D cases, respectively, and generate data at $t = 0.5,0.6,...,2.0$, i.e., 16 snapshots in total. The input radius $r$ of the influence function $\phi(r)$ is clipped to be at least $r_{\text{min}}=0.01$, both in simulation and learning, to avoid the singularity. While using the same number of particles and make observations at the same time instants as in \cite{mao2019nonlocal}, our method does not require knowledge of the initial condition or the velocity field in the data snapshots, as opposed to their method based on Bayesian optimization to infer $\alpha$. 

We apply the sliced Wasserstein distance for the distribution loss, with the batch size equal to the number of particles in training data. When calculating the distance at each time instant, the generated and real distributions are normalized with the mean and standard deviation of the real distribution. For the Newton loss, we set $K = 16$, $M = 10000$ and $B = 10$. We use $2\times \text{sigmoid}(\beta)$ to represent the inferred $\alpha$ so that the inference is bounded by 0 and 2, and the variable $\beta$ is initialized as 0.

\begin{figure}[H]
    \centering
    \begin{subfigure}{0.4\textwidth}
        \centering
        \includegraphics[width = \textwidth]{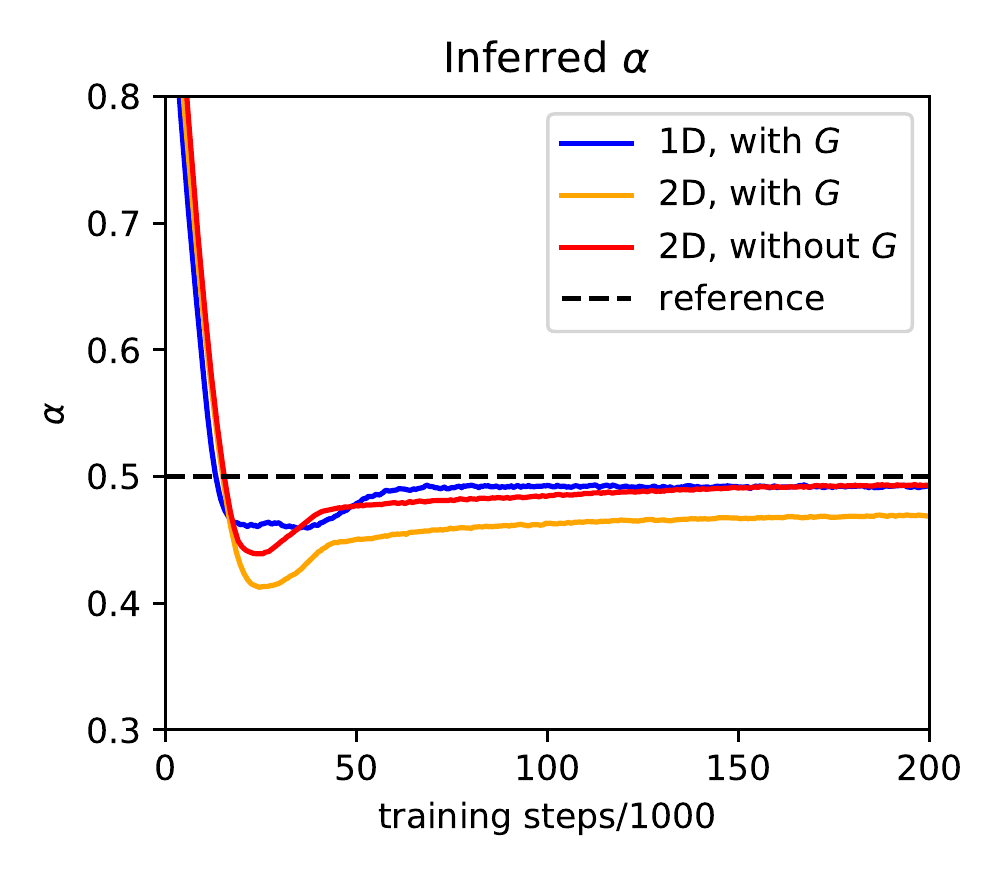}
        \caption{}
    \end{subfigure}
    \begin{subfigure}{0.4\textwidth}
        \centering
        \includegraphics[width = \textwidth]{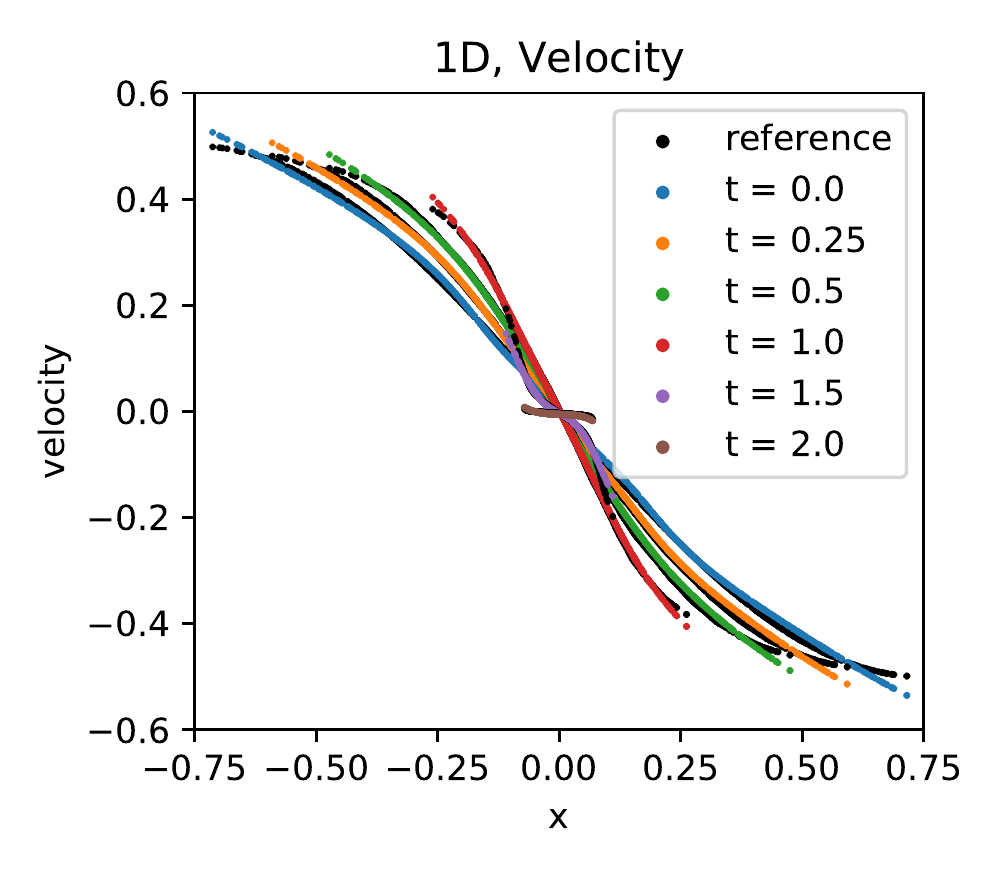}
        \caption{}
    \end{subfigure}
    
    \begin{subfigure}{0.8\textwidth}
    \centering
    \includegraphics[width =0.32\textwidth]{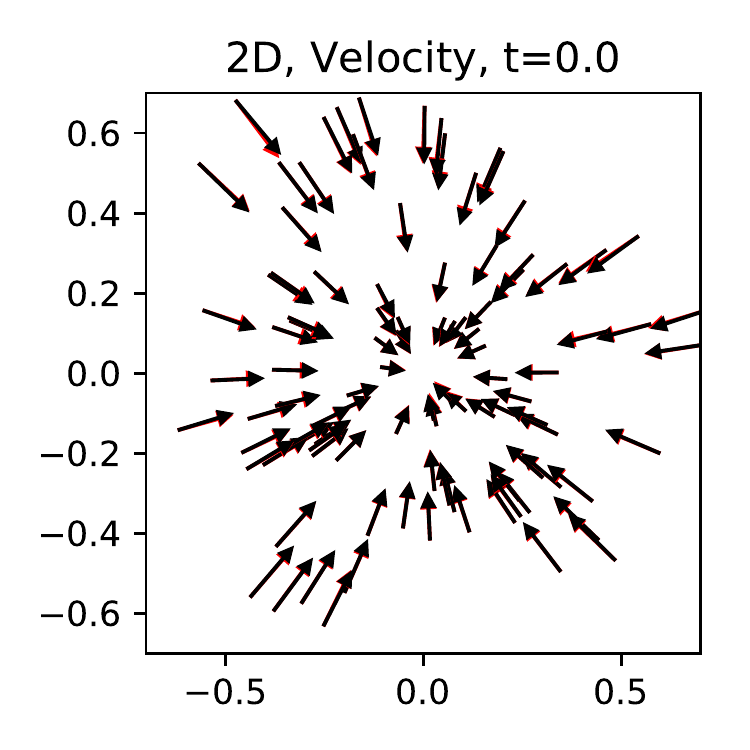}
    \includegraphics[width =0.32\textwidth]{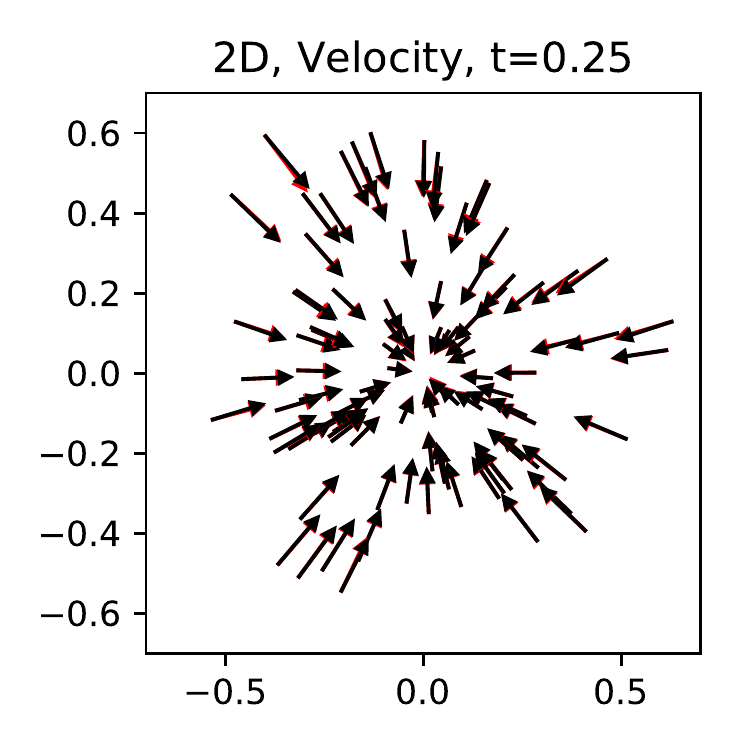}
    \includegraphics[width =0.32\textwidth]{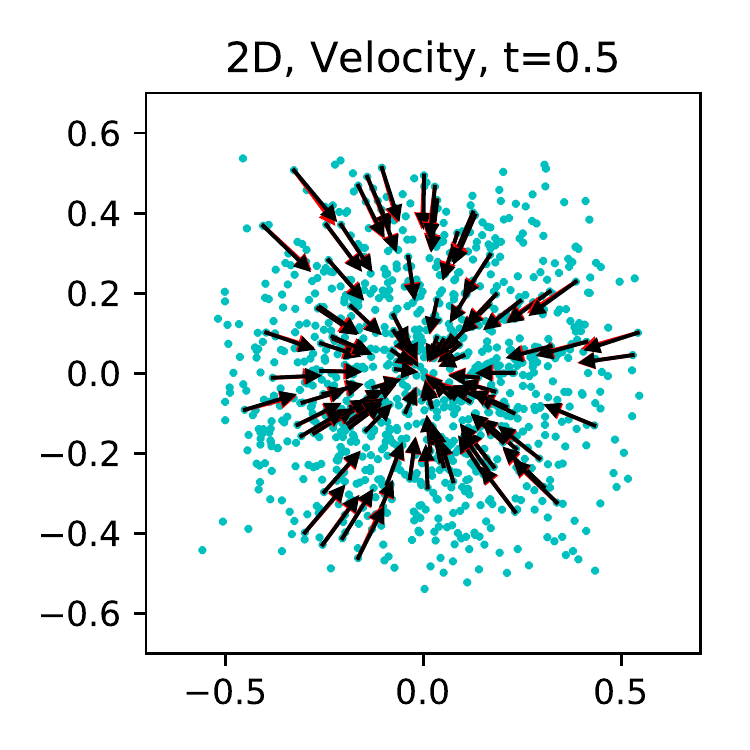}
    
    \includegraphics[width =0.32\textwidth]{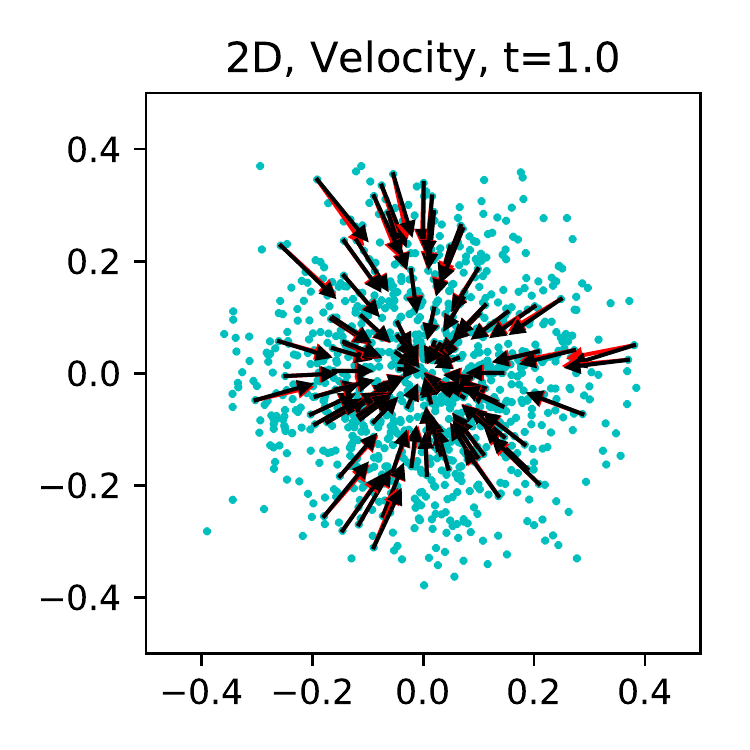}
    \includegraphics[width =0.32\textwidth]{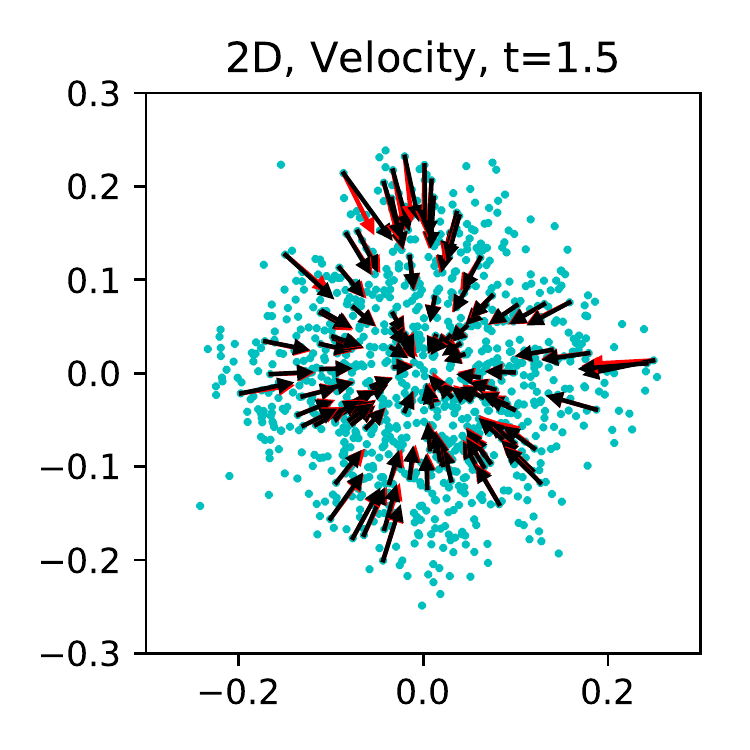}
    \includegraphics[width =0.32\textwidth]{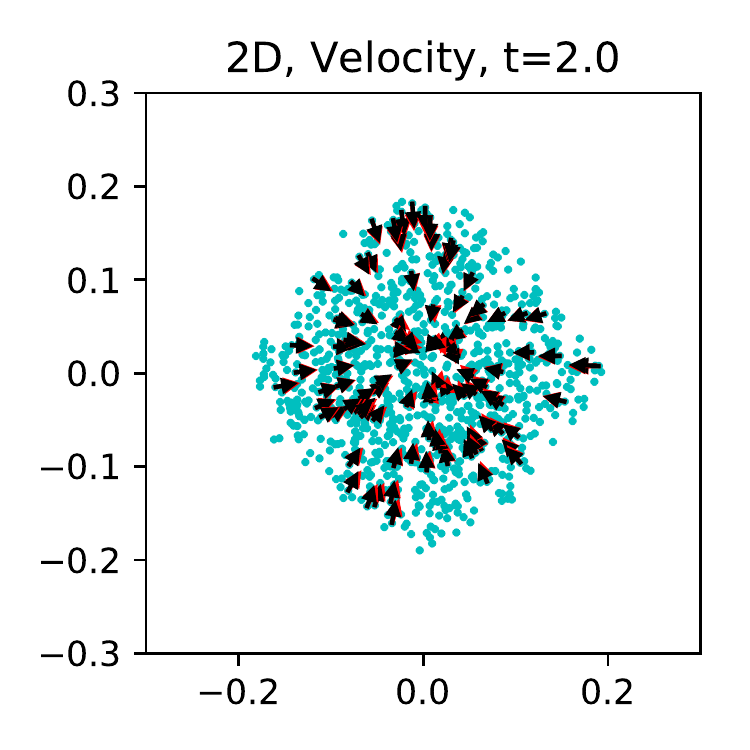}
    \caption{}
    \end{subfigure}
    
    \caption{Results for the 1D and 2D problems of interacting particle systems. (a): The inferred alpha in 1D and 2D problems. (b): The inferred velocity field against the reference velocity from simulation in the 1D problem. The dots also show the distribution of training data. (c): The inferred velocity field (red arrows) against the reference velocity from simulation (black arrows) in the 2D problem. For each arrow, the length represents the norm of the velocity scaled by $0.3$. The green dots are random samples from training data.}
    \label{fig:Interacting}
\end{figure}

The results are visualized in Figure~\ref{fig:Interacting}. For the 1D case, the inferred $\alpha$ is 0.492 with ground truth 0.5 in the end of training. As a comparison, the inference is 0.480 for the 1D case in \cite{mao2019nonlocal}. Our inferred velocity field also matches well with the reference ground truth at different time instants, even at $t= 0$ and $0.25$ when we have no observations at all. This is because we incorporated the knowledge of dynamics in the learning system. For the 2D case, while the inferred velocity field also matches well with the ground truth, the inferred $\alpha$ is 0.469 with ground truth 0.5 in the end of training, i.e., about 6\% error, much larger than in the 1D case. We attribute the error to the fact that the singularity problem for $\phi(r)$, where the order for $r$ is dimension-dependent, is more severe in the 2D case. In supplementary information section S4, we show that as we increase $r_{\min}$ from 0.01 to 0.1, the error is reduced from 6\% to 2\%. Alternatively, in the 2D case we also try to directly use the data at $t = 0.5$ as the starting coordinates to generate trajectories for $t\ge 0.5$ ($M$ is also reduced to the data size $9976$), so that the initial generator $G$ is removed from the generative model.   By doing so, the number of trainable variables is reduced and the learning becomes easier. The inferred $\alpha$ is 0.493 in the end of training. As a comparison, the inference is 0.513 for the 2D case in \cite{mao2019nonlocal}. However, we remark that while this strategy helps to infer $\alpha$ without extra data, we cannot infer the velocity or density for $0\le t < 0.5$.

\section*{Paired Observations}\label{sec:discuss}

We should note that the convergence of the marginal distribution in each snapshot does not necessarily lead to the convergence of the joint distribution of coordinate tuples $(\bm{X}_{t_1}, \bm{X}_{t_2}, ..., \bm{X}_{t_n})$. While the joint distribution is not available for unpaired observations that we are focused on so far, in other cases where the observed particle coordinates can be paired across snapshots, it is possible to improve the inference by fitting the joint distribution of coordinate tuples with the corresponding generated joint distribution. 

In \cite{ding2020subadditivity}, the authors pointed out that Wasserstein convergence of the distribution of $(\bm{X}_{t_1}, \bm{X}_{t_2}, ..., \bm{X}_{t_n})$ is equivalent to Wasserstein convergence of the distributions of  $(\bm{X}_{t_i}, \bm{X}_{t_{i+1}})$ for all $i = 1,2,...,n-1$, if $(\bm{X}_{t_1}, \bm{X}_{t_2}, ..., \bm{X}_{t_n})$ is a Markov chain. However, the sample spaces for $\{\bm{X}_{t_i}\}_{i=1}^n$ are limited to be finite and discrete in \cite{ding2020subadditivity}, thus the result doesn't apply to dynamic systems in continuous spaces. Here, we present a new theorem for continuous sample spaces:

\begin{theorem}\label{thm:1}
Let $(X_1, X_2,...X_T)$ be a Markov chain of length $T\ge3$ and we use $X_{i:j}$ to denote the nodes $(X_i, X_{i+1}...X_j)$, for $i\le j$. Suppose the domain $D_t$ for $X_t$ is a compact subset of $\mathbb{R}^{d_t}$ for $t=1,2...T$. We use the $l_q$ ($q\ge 1$) Euclidean metric for all the Euclidean spaces with different dimensions. 

Let $\{P_n^{X_{i:j}}\}_{n=1}^\infty$ and $P^{X_{i:j}}$ be probability measures of $X_{i:j}$ for $i\le j$, $P_n^{X_i|X_j}$ and $P^{X_i|X_j}$ be the corresponding probability transition kernels. If $P_n^{X_{t:t+1}}$ converges to $P^{X_{t:t+1}}$ in Wasserstein-$p$ ($p\ge 1$) metric for all $t = 1,2...T-1$, $P_n^{X_t|X_{t+1}}$ and $P^{X_{t+2}|X_{t+1}}$ as functions of $X_{t+1}$ are $C$-Lipschitz continuous in Wasserstein-$p$ metric for all $t=1,2...T-2$ and $n$, where $C$ is a constant, then $P_n^{X_{1:T}}$ converges to $P^{X_{1:T}}$ in Wasserstein-$p$ metric. 
\end{theorem}

We present the proof for Theorem~\ref{thm:1} in the supplementary information section S5. The assumption on the continuity of the probability transition kernels is not required for finite discrete sample spaces in~\cite{ding2020subadditivity}, but the theorem does not hold without it for continuous sample space. We provide a counter-example in the supplementary information section S6.

The theorem states that under certain conditions, we can set our goal as fitting the distributions of  $(\bm{X}_{t_i}, \bm{X}_{t_{i+1}})$, i.e., coordinate pairs from adjacent snapshots. This should be easier compared with directly fitting the distribution of $(\bm{X}_{t_1}, \bm{X}_{t_2}, ..., \bm{X}_{t_n})$, since the effective dimensionality is reduced. We still can view this approach as ``ensemble-regression'', except that instead of a curve, we try to fit the data with a {\em 2D surface} $(t_i, t_j) \rightarrow \rho_{t_i,t_j}$ in the probability measure space, where $\rho_{t_i, t_j}$ denotes the joint distributions of $(\bm{X}_{t_i}, \bm{X}_{t_j})$. 


\begin{figure}[ht]
    \centering
    \includegraphics[width = 0.45\textwidth]{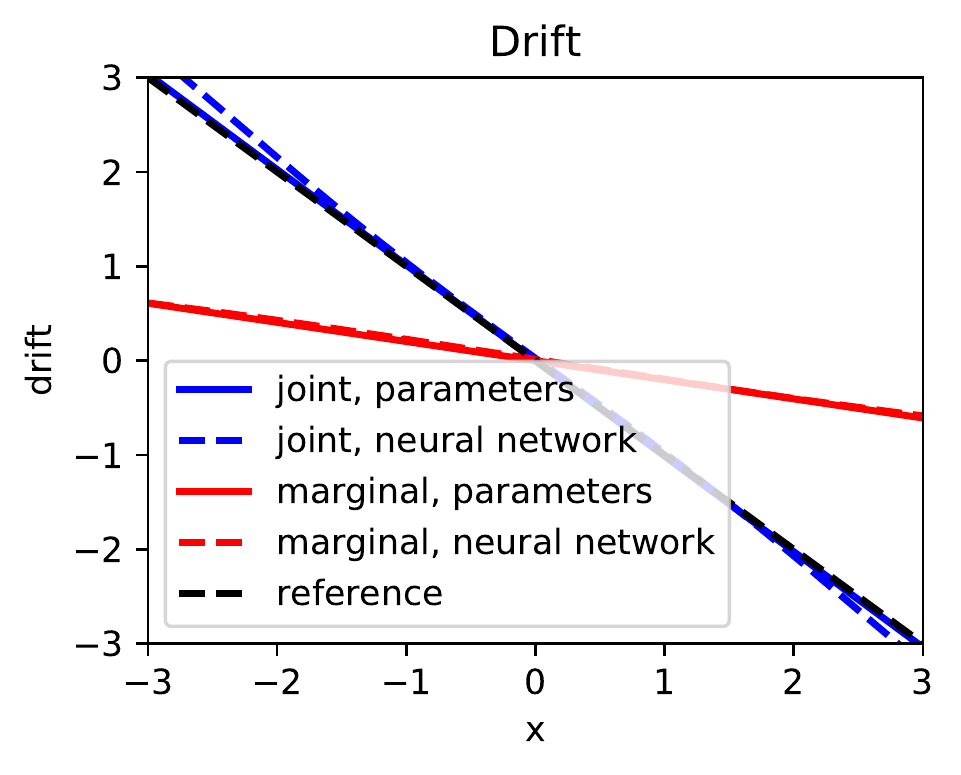}
    \includegraphics[width = 0.45\textwidth]{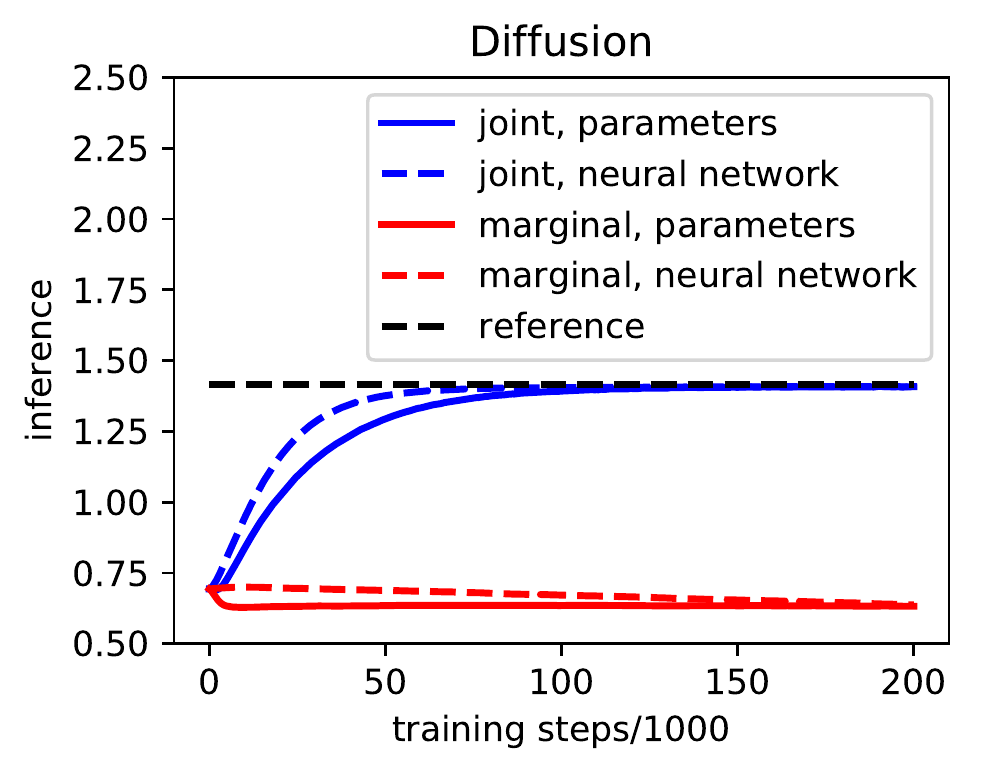}

    \caption{Inferred drift function in the end of training and diffusion coefficient during training for the OU process problem, using a linear function or a neural network to parameterize the drift function.}
    \label{fig:OU}
\end{figure}

As an illustration, we study the 1D Ornstein–Uhlenbeck (OU) process:
\begin{equation}\label{eqn:OU}
\begin{aligned}
     dX_t = -X_tdt + \sqrt{2}dW_t,
\end{aligned}
\end{equation}
with $\rho_0 = \mathcal{N}(0,1)$, so that $\rho_t = \mathcal{N}(0,1)$ for any $t>0$. This is a special example as the governing SODE is not unique given $\rho_t$. We make observations of 100 particles at $t=0.1,0.2,0.3,0.4,0.5,0.75,1.0$ as data, and compare the inferences by fitting the marginal distributions of individual coordinates or the joint distributions of adjacent coordinate pairs using the SW distance. The drift function is parameterized by a linear function $y(x) = ax$ or a neural network, while the diffusion coefficient is represented by a trainable variable rectified by a softplus function. We present the results in Figure~\ref{fig:OU}, where we can clearly see the failure in the cases of fitting the marginal distributions, while fitting the joint distributions works very well.

\section*{Summary and Discussion}\label{sec:summary}
We have proposed a new method for inferring the governing dynamics of particles from {\em unpaired observations} of their coordinates at multiple time instants, namely ``snapshots''. We fitted the observed particle ensemble distribution with a physics-informed generative model, which can be viewed as performing regression in the probability measure space. We refer to this approach as generative ``ensemble-regression'', in analogy to the classic ``point-regression'', where we infer the dynamics by performing regression in the Euclidean space. 

We first applied the method to particle systems governed by independent stochastic ordinary differential equations (SODE) with Brownian or L\'evy noises, where we inferred the drift and the diffusion terms from a small number of snapshots. In the L\'evy noise case, we demonstrated that the heavy tails in the distributions could spoil the training, but we addressed this issue by applying a preprocessing map to both the generated and target distributions. In scenarios with noisy or truncated training data, we  modified the generated distributions accordingly by perturbing or filtering the generated samples. We then addressed high-dimensional SODE problems using the adversarial loss in GANs. In the end, we managed to learn the parameters for particle interactions in a nonlocal flocking systems.

It is possible to apply our method to learn the interaction parameters for particle-based simulation methods. In particular, we will fit the target mass and velocity distributions coming from the analytical solution or other simulation methods that are accurate but expensive. We leave this promising research direction for future work.

\subsection*{Acknowledgement}

This work was supported by the PhILMS grant DE-SC0019453 and by the OSD/AFOSR MURI Grant FA9550-20-1-0358. We would like to thank Prof. Hui Wang and Ms. Tingwei Meng for carefully checking the proof of our theorem. We also want to thank Dr. Zhongqiang Zhang for helpful discussions.

\section*{Supplementary Information}

\subsection*{S1. Comparison between Sliced Wasserstein Distance and WGAN-GP}\label{sec:SWD}

The squared sliced Wasserstein-2 distance is the expectation of the squared Wasserstein-2 distance between the two input measures projected onto uniformly random directions. To estimate $SW_2(\mu, \nu)$ from samples of $\mu$ and $\nu$, we use the following process introduced in \cite{deshpande2018generative}.

\begin{itemize}
    \item Draw samples independently from $\mu$ and $\nu$ with batch size $b$, denoted as $\mathcal{U}$ and $\mathcal{V}$. 
    \item Uniformly sample $m$ projection 
    directions $\{\bm{e}_j\}_{j=1}^m$ in $\mathbb{R}^d$. In this paper we set $m = 1000$.
    \item For each random direction $\bm{e}_j$, project and sort the samples in $\mathcal{U}$ and $\mathcal{V}$ in the direction of $\bm{e}_j$, getting $\{u_{i,j}\}_{i=1}^b$ and $\{v_{i,j}\}_{i=1}^b$, where $u_{i,j}\le u_{i+1,j}$ and $v_{i,j}\le v_{i+1,j}$ for $i = 1,2...b-1$. Calculate $L_j = \sum_{i=1}^b(u_{i,j}-v_{i,j})^2/b$.
    \item Calculate $L =\sum_{j=1}^m L_j/m$ as the estimation of squared $SW_2(\mu, \nu)$.
\end{itemize}

Compared with GANs, the sliced Wasserstein distance does not need discriminators, and is more robust than WGAN-GP in low dimensional problems. Take the following 1D problem as an example. We consider the SODE:
\begin{equation}
    \begin{aligned}
    dX_t = a dt &+ b dB_t, \quad t \ge 0, 
    \end{aligned}
\end{equation}
with $\rho_0 = \mathcal{N}(-0.5, 0.5)$. We set $a = b = 1$ so that the exact solution is 
\begin{equation}
    \begin{aligned}
    \rho_t = \mathcal{N}(t - 0.5, t + 0.5).
    \end{aligned}
\end{equation}
We have $10,000$ samples at $t = 0.5$ and $1.5$, respectively, as the training data, and wish to infer the constant drift and diffusion coefficients. The input noise to the generator is uniform noise from -1 to 1. The results are shown in Figure~\ref{fig:Inverse1DSimple}.  

\begin{figure}[ht]
    \centering
    \begin{subfigure}{0.45\textwidth}
        \includegraphics[width = \textwidth]{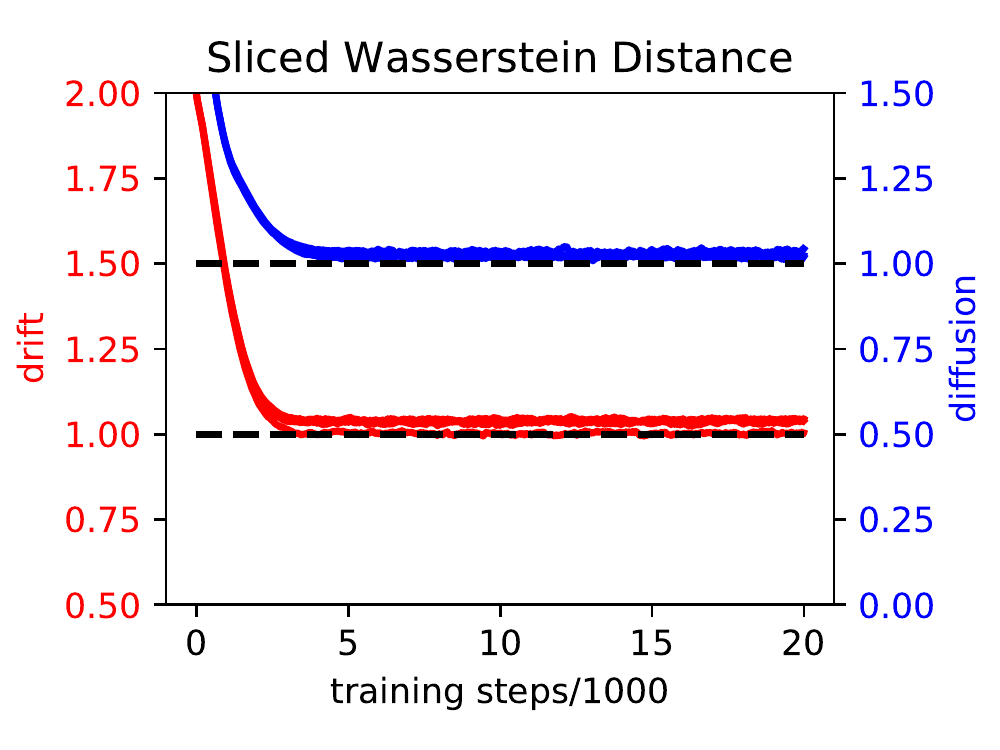}
        \caption{}
    \end{subfigure}
    \begin{subfigure}{0.45\textwidth}
        \includegraphics[width = \textwidth]{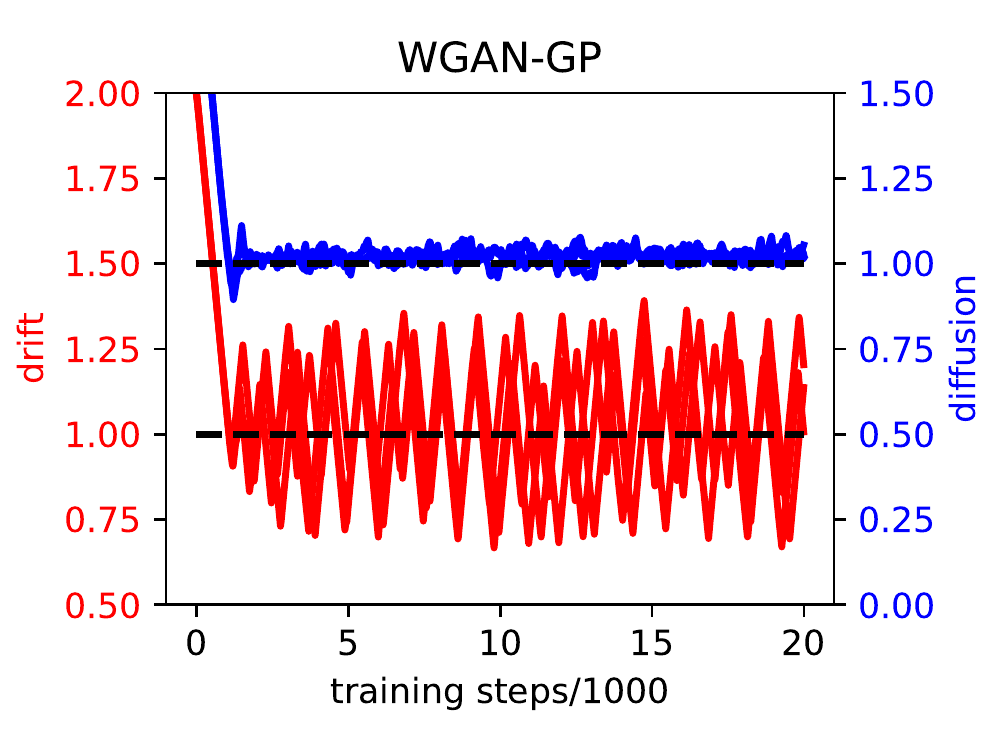}
        \caption{}
    \end{subfigure}
    \caption{Inferred drift and diffusion coefficient during the training, using (a) the squared sliced Wasserstein-2 distance or (b) WGAN-GP for $\textsf{d}$. The black dashed lines represent the ground truth, while the multiple colored lines represent the results in three independent runs.}
    \label{fig:Inverse1DSimple}
\end{figure}
We can clearly see oscillations of the inferred drift coefficient for WGAN-GP. This can be attributed to the two-player game between the generator and discriminator, which was also reported in \cite{daskalakis2017training}. 

\begin{figure}[ht]
    \centering
    \begin{subfigure}{0.48\textwidth}
    \centering
        \includegraphics[width = \textwidth]{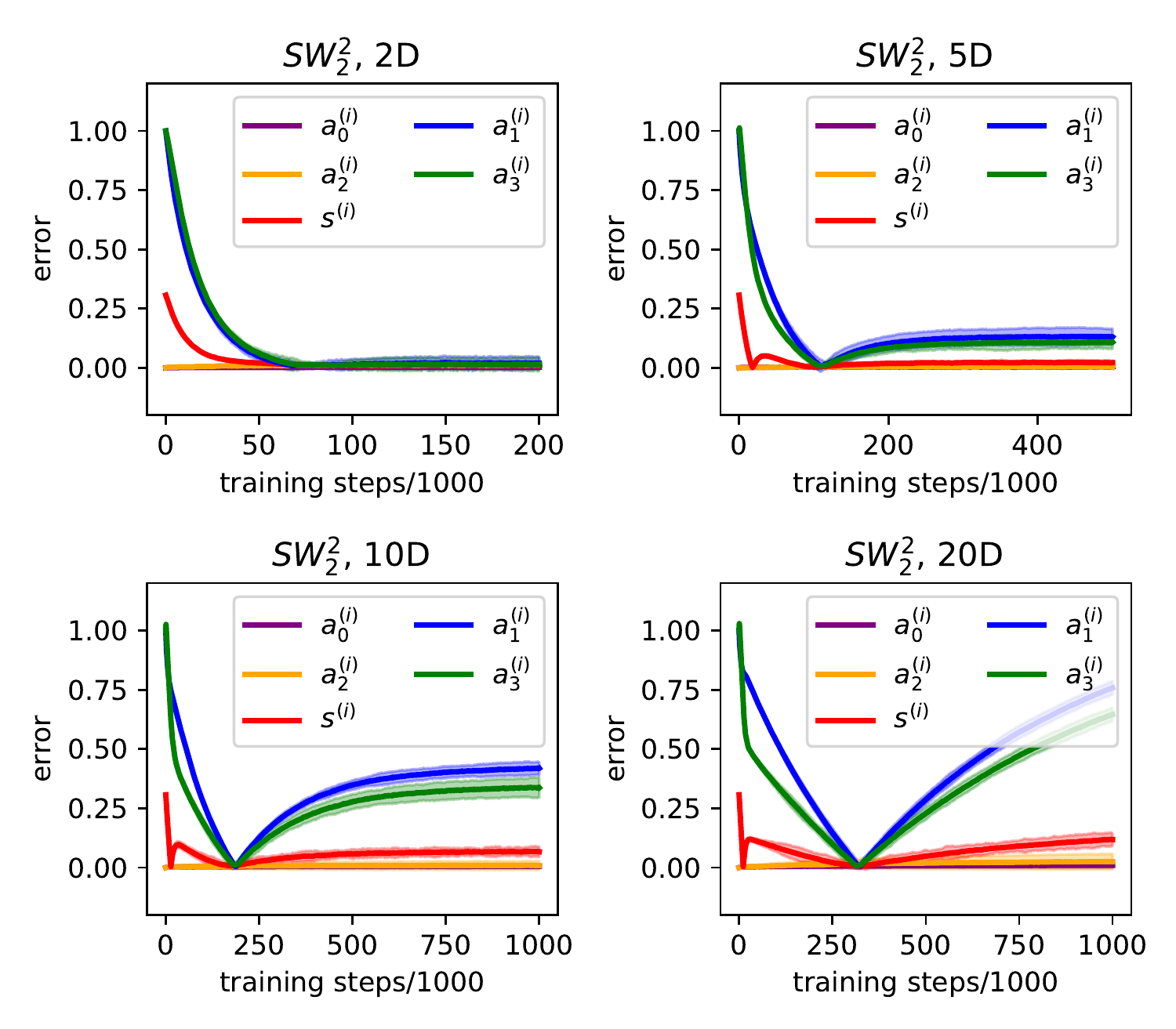}
        \caption{}
        \label{fig:20D-1}
    \end{subfigure}
    \begin{subfigure}{0.48\textwidth}
    \centering
    \includegraphics[width = \textwidth]{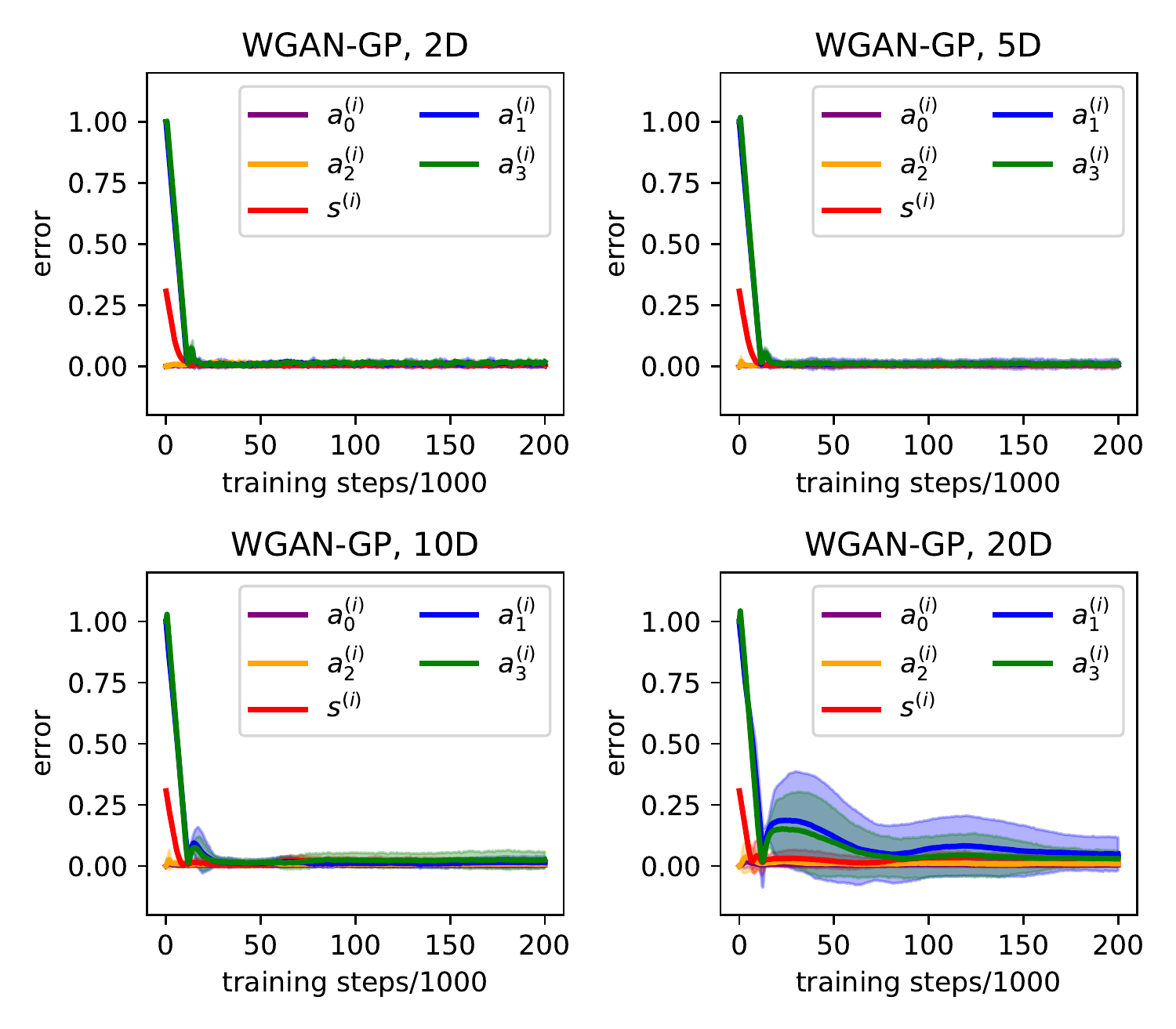}
        \caption{}
        \label{fig:20D-2}
    \end{subfigure}
    \caption{Results of 2D, 5D, 10D and 20D problems, using (a) the squared sliced Wasserstein-2 distance or (b) WGAN-GP for $\textsf{d}$. The solid lines and shaded areas represent the mean and two standard deviations of the errors over all dimensions.}
    \label{fig:Inverse20D}
\end{figure}

We did not observe significant oscillations using WGAN-GP in higher dimensional problems. Indeed, WGAN-GP outperforms the sliced Wasserstein distance in high dimensional problems. Let us consider the following simple SODE problem as an example, where the motions are uncoupled between dimensions:

\begin{equation}\label{eqn:HighInv}
    \begin{aligned}
    dX_t^{(i)} = (X_t^{(i)} - (X_t^{(i)} )^3)dt + dB_t^{(i)}, i = 1,2,...d
    \end{aligned}
\end{equation}
where $X_t^{(i)}$ is the $i$-th component of $\bm{X}_t \in \mathbb{R}^d$, with $\rho_0 =  \mathcal{N}(\bm{0},0.04\bm{I}_d)$. We prepare $10^5$ sample paths and observe all the particle positions at $t = 0.2, 0.5, 1.0$ as our training data. We use a cubic polynomial $a^{(i)}_0 + a^{(i)}_1x + a^{(i)}_2x^2 + a^{(i)}_3 x^3$ to parameterize the $i$-th component of the drift. We use another trainable variable rectified by a softplus function to represent the diffusion coefficient $s^{(i)} = 1$ in each direction. In total, $5d$ variables are used to parameterize the $d$-dimensional SODE.
The results are shown in Figure~\ref{fig:Inverse20D}. While the SW distance works for 2D problems, it does not scale well to high dimensional problems for which WGAN-GP gives much better inferences.

\subsection*{S2. A Study on the Density Estimation}\label{sec:forward}
In this section, we wish to show that our method can make use of the SODE and multiple snapshots to reduce the error of density estimation from limited samples. We consider the one-dimensional SODE:
\begin{equation}\label{eqn:1D}
    \begin{aligned}
    dX_t = (4X_t - {X_t}^3) dt &+  0.4dB_t, \quad t \ge 0, 
    \end{aligned}
\end{equation}
with $\rho_0 = 0.5 \mathcal{N}(-0.5,0.3^2) + 0.5 \mathcal{N}(0.5,0.3^2)$.

We test the following two cases of training data sets. In both cases we have 1000 samples at $t = 0.05, 0.1, 0.15, 0.2, 0.25$, but the particle trajectories are different.
\begin{itemize}
    \item Case 1: We make observations from different sets of trajectories at different time instants.
    \item Case 2: We make observations from the same set of trajectories at different time instants.
\end{itemize}

We assume that we know Equation~\ref{eqn:1D}, but have no knowledge of $\rho_0$. We could make the analogy of performing linear regression, where we know the slope but do not know the intercept. In Figure~\ref{fig:Forward2} we show the comparison between the inferred densities and the densities estimated directly from data. All the densities are estimated using Gaussian kernel density estimation. The inferred densities and the ground truth come from $10^5$ samples produced by the generative model or simulation, with bandwidth decided by the Scott's rule. To remove the effect of bandwidth selection in the density estimation directly from data, we perform a grid search of the optimal bandwidth factor via the $L_2$ error against ground truth, from 0.01 to 0.5 with grid size 0.01 (Scott's rule suggests about 0.25). 
%

\begin{figure}[H]
    \centering
    \begin{subfigure}{0.45\textwidth}
		\includegraphics[width = \textwidth]{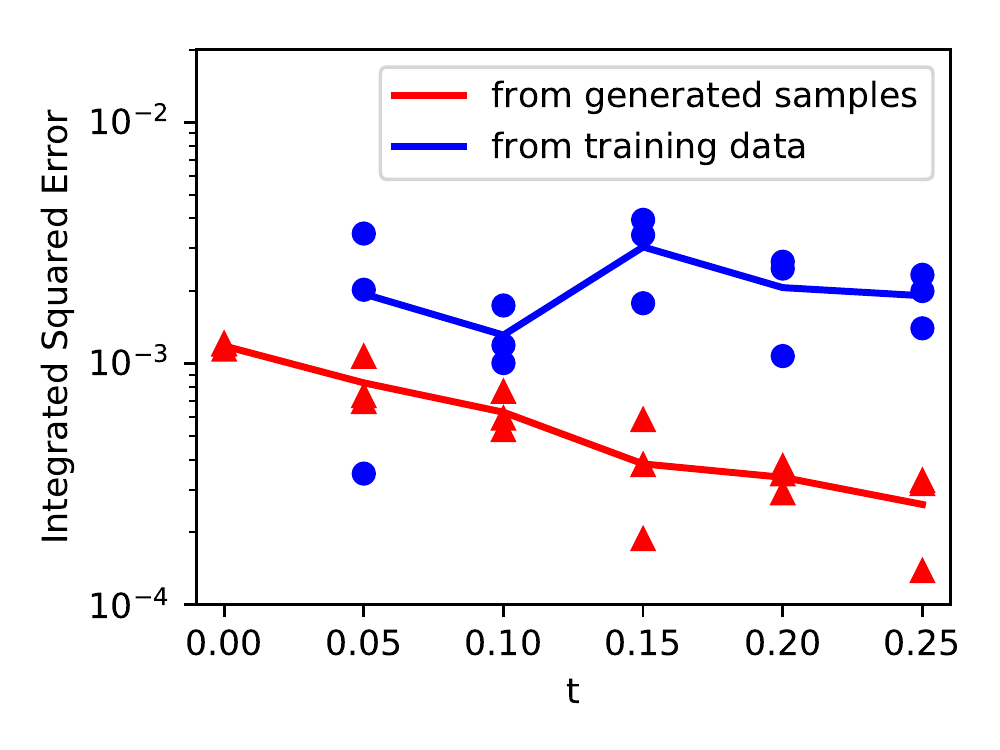}
		\caption{}
		\label{fig:forward-3}
	\end{subfigure}
	\begin{subfigure}{0.45\textwidth}
		\includegraphics[width = \textwidth]{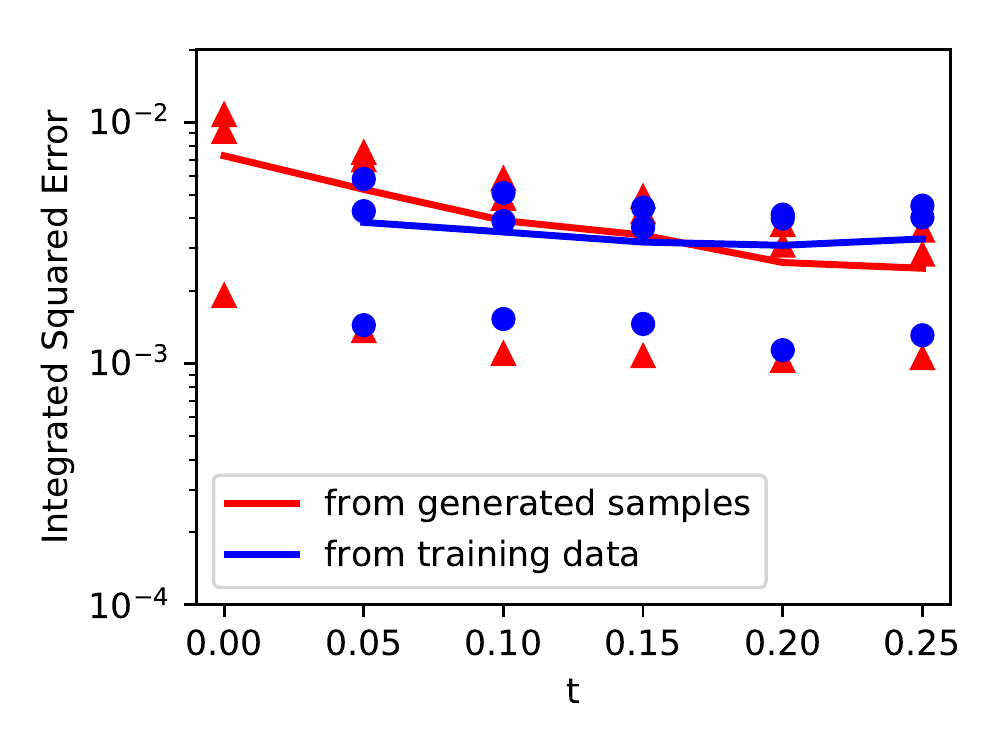}
		\caption{}
		\label{fig:forward-4}
	\end{subfigure}
    \caption{Integrated squared errors (ISE) of the inferred density functions at different time instants in (a) case 1 and (b) case 2. The squared errors are integrated from -4 to 4. The colored lines show the ISE averaged over three independent runs, while the markers with the corresponding colors show the ISE in each run.}
    \label{fig:Forward2}
\end{figure}

The grid search strategy is actually infeasible in practice since we don't know the ground truth, but it should perform better than any bandwidth selector. Despite that, in Figure~\ref{fig:forward-3} we can clearly see that in case 1 our inferred densities with naive Scott's rule significantly outperform the ones estimated directly from training data with grid search. This cannot be attributed to the number of samples, since in case 2 the inferred densities with our method cannot perform better, as shown in Figure~\ref{fig:forward-4}.

The different performances in two cases are reasonable. In case 1, our method can utilize the SODE and the observations at multiple time instants to reduce the error from limited samples. As we infer the density at $t = 0.05$,  we are not  only utilizing the data at $t= 0.05$ but also implicitly taking advantage of the data at later time instants. In case 2, since the observations come from the same sample paths, the observations at later time instants cannot provide additional information considering that the SODE is already known. Let us again make an analogy in the context of linear regression: if the observations have independent noise, multiple observations will be more helpful than a single observation. However, if the observations have the same noise, multiple observations cannot help us more than a single observation, if we already know the slope.

\subsection*{S3. Training Data in 2D Problems with Various Scenarios of Observations}\label{sec:2Dtrainingdata}
In Figure~\ref{fig:2Dtrainingdata} we visualize the training data in 2D problems with various scenarios of observations.
\begin{figure}[ht]
    \centering
    \includegraphics[width = 0.8\textwidth]{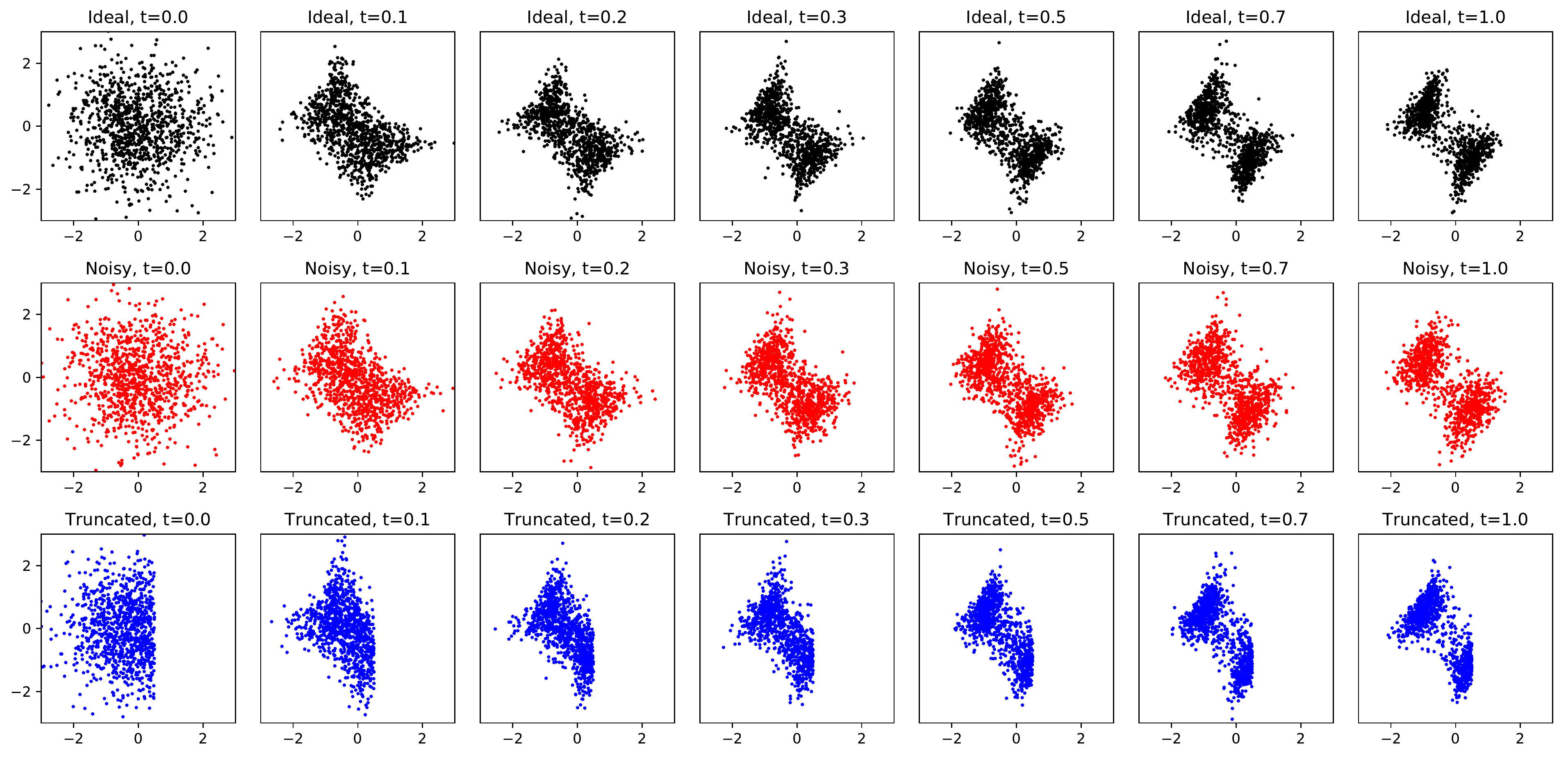}
    \caption{Samples from the training data in 2D problems with ideal, noisy, or truncated observations.}
    \label{fig:2Dtrainingdata}
\end{figure}

\subsection*{S4. Effect of Clipping Radius in the Interacting Particle System}

\begin{figure}[ht]
    \centering
    \includegraphics[width = 0.6\textwidth]{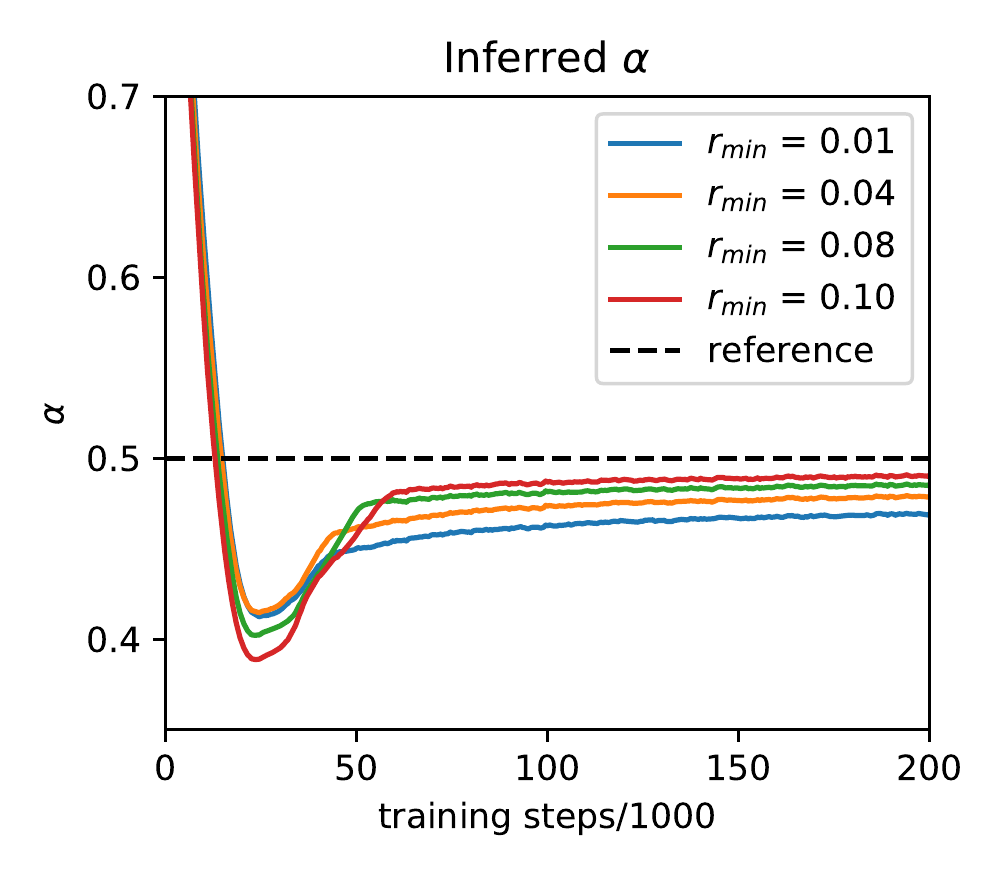}
    \caption{Inferred $\alpha$ in the 2D interacting particle system during training, with different clipping radius $r_{\text{min}}$.}
    \label{fig:clipradius}
\end{figure}

In this section we study the effect of various clipping radius $r_{\text{min}}$ from 0.01 to 0.1, both in the simulation for data generation and the learning algorithm. We remark that changing $r_{\text{min}}$ in this range would not make a huge difference to the system behavior, since the dynamics are dominated by non-local interactions.
In Figure~\ref{fig:clipradius} we show the inferred $\alpha$ during the training. As $r_{\text{min}}$ increase from 0.01 to 0.1, the inferred $\alpha$  in the end of training increases from 0.469 to 0.490, with the ground truth 0.5. This suggests that the error mainly comes from the singularity problem of $\phi(r)$ when $r$ is small.

\subsection*{S5. Proof for Theorem 1}\label{sec:proof}

The proof is based on the weak convergence: convergence of a sequence of probability measures $\{P_n\}_{n=1}^\infty$ to a probability measure $P$ in the Wasserstein-$p$ metric is equivalent to the weak convergence plus the convergence of $p$-th moments of $\{P_n\}_{n=1}^\infty$ to $P$~\cite{villani2008optimal}. Here, the weak convergence mean $E_{P_n}[f(x)]$ converges to $E_{P}[f(x)]$ for each $f \in \mathcal{F}$, where $\mathcal{F} = \{\text{all bounded and continuous functions}\}$ or $\mathcal{F} = \{\text{all bounded and Lipschitz functions}\}$, which are equivalent\cite{klenke2013probability}. Note that in Theorem 1 we assume that the domain for $X_t$ is a compact subset of Euclidean space for each $t$, thus $\|X_{1:T}\|^p$ is a bounded and continuous function of $X_{1:T}$. The convergence of $p$-th moment then directly comes from the weak convergence, i.e., we only need to show that $\{P_n\}_{n=1}^\infty$ converges to $P$ weakly. This is proved in Lemma~\ref{weakconvergence} below.

The notations are inherited from Theorem 1. For simplicity, we will sometimes use $X,Y,Z$ to represent the nodes in sequence, and use $P^{XY}$ to represent the probability measure of $(X,Y)$, etc.

\begin{lemma}\label{integralLip}
Let $P_y$ mapping from the Euclidean space to the probability measure space be $C$-Lipschitz in Wasserstein-$p$ metric ($p\ge 1$). For any $f(x,y,z): \mathbb{R}^{d_x} \times \mathbb{R}^{d_y} \times \mathbb{R}^{d_z}\rightarrow \mathbb{R}$ $K$-Lipschitz,
$$g(x,y) = \int f(x,y,z) dP_y(z)$$
is $(C+1)K$-Lipschitz.
\end{lemma}
\begin{proof}
Note that
\begin{equation}
\begin{aligned}
   & \left| g(x+\epsilon, y+\xi) - g(x, y) \right| \\
   =& \left|\int f(x+\epsilon, y+\xi, z) dP_{y+\xi}(z)-  \int f(x, y, z) dP_{y}(z)\right| \\
   \le& \left|\int f(x+\epsilon, y+\xi, z) dP_{y+\xi}(z)-  \int f(x+\epsilon,y+\xi, z) dP_{y}(z) \right| +\\
   &\left|\int f(x+\epsilon, y+\xi, z) dP_{y}(z)-  \int f(x, y, z) dP_{y}(z)\right|.
\end{aligned}
\end{equation}

Firstly
\begin{equation}
\begin{aligned}
   & \left|\int f(x+\epsilon, y+\xi, z) dP_{y+\xi}(z)-  \int f(x+\epsilon,y+\xi, z) dP_{y}(z) \right|\\
  \le&  K W_1(P_y, P_{y+\xi})\\
  \le&  K W_p(P_y, P_{y+\xi})\\
  \le&  CK\|\epsilon\|,
\end{aligned}
\end{equation}
where the first inequality comes from the Kantorovich–Rubinstein formula~\cite{villani2008optimal} and that $f(x+\epsilon, y+\xi, z)$ is $K$-Lipschitz. The last inequality comes from that $P_y$ is $C$-Lipschitz in Wasserstein-$p$ sense.

Secondly,
\begin{equation}
\begin{aligned}
   & \left|\int f(x+\epsilon, y+\xi, z) dP_y(z)-  \int f(x, y, z) dP_y(z)\right|\\
   \le & \int \left|f(x+\epsilon, y + \xi, z) - f(x, y, z) \right| dP_y(z) \\
   \le & \int K\|(\epsilon, \xi)\| dP_y(z) \\
   = & K\|(\epsilon, \xi)\|.
\end{aligned}
\end{equation}

We conclude that 
\begin{equation}
\begin{aligned}
\left| g(x+\epsilon, y+\xi) - g(x, y) \right| \le (C+1)K\|(\epsilon,\xi)\|,
\end{aligned}
\end{equation}
i.e. $g(x,y)$ is $(C+1)K$-Lipschitz
\end{proof}

\begin{lemma}\label{marginalweakconvergence}
If $P_n^{XY}$ converge to $P^{XY}$ weakly, then $P_n^X$ converge to $P^X$ weakly.
\end{lemma}
\begin{proof}
For any bounded and continuous function $f(x)$
\begin{equation}
\begin{aligned}
 &\int f(x)P_n^X(x) - \int f(x)P^X(x) \\
=&\iint f(x)P_n^{XY}(x,y) - \iint f(x)P^{XY}(x,y) \rightarrow 0.
\end{aligned}
\end{equation}
The convergence comes from the weak convergence of $P_n^{XY}$ and that $f(x)$ is bounded and continuous as a function of $x$ and $y$.
\end{proof}

\begin{lemma}\label{longlemma}
Suppose $P_n^{YZ}$ converge to $P^{YZ}$ weakly, $P^{Z|Y = y}$ is Lipschitz in Wasserstein-$p$ metric. For any $g(y,z): \mathbb{R}^{d_y} \times \mathbb{R}^{d_z}\rightarrow \mathbb{R}$ bounded and Lipschitz, we have 
\begin{equation}
\begin{aligned}
   & \iint g(y,z) dP_n^{Z|Y=y}(z) dP_n^{Y}(y) \\
   - & \iint g(y,z) dP^{Z|Y=y}(z) dP_n^{Y}(y) \rightarrow 0.
\end{aligned}
\end{equation}
\end{lemma}

\begin{proof}

Since $g(y,z)$ is bounded and continuous, and $P_n^{YZ}$ converge to $P^{YZ}$ weakly,
\begin{equation}
\begin{aligned}
   \iint g(y,z) dP_n^{YZ}(y,z) -\iint g(y,z) dP^{YZ}(y,z) \rightarrow 0,
\end{aligned}
\end{equation}
i.e.
\begin{equation}\label{eqn:2-1}
\begin{aligned}
   & \iint g(y,z) dP_n^{Z|Y=y}(z) dP_n^{Y}(y) \\
   - & \iint g(y,z) dP^{Z|Y=y}(z) dP^{Y}(y) \rightarrow 0.
\end{aligned}
\end{equation}

Since $g(y,z)$ is bounded and Lipschitz, $P^{Z|Y=y}$ is Lipschitz, we have 
\begin{equation}
\begin{aligned}
\int g(y,z) dP^{Z|Y=y}(z)
\end{aligned}
\end{equation}
is Lipschitz. Further since $P_n^{Y}$ converge to $P^{Y}$ weakly (from Lemma~\ref{marginalweakconvergence}):
\begin{equation}\label{eqn:2-2}
\begin{aligned}
   & \iint g(y,z) dP^{Z|Y=y}(z) dP_n^{Y}(y) \\
   - & \iint g(y,z) dP^{Z|Y=y}(z) dP^{Y}(y) \rightarrow 0.
\end{aligned}
\end{equation}

Combining \ref{eqn:2-1} and \ref{eqn:2-2}, we have
\begin{equation}
\begin{aligned}
   & \iint g(y,z) dP_n^{Z|Y=y}(z) dP_n^{Y}(y) \\
   - & \iint g(y,z) dP^{Z|Y=y}(z) dP_n^{Y}(y) \rightarrow 0.
\end{aligned}
\end{equation} 

\end{proof}

\begin{lemma}\label{weakconvergence}
Let $(X_1, X_2,...X_T)$ be a Markov chain of length $T\ge3$ and we use $X_{i:j}$ to denote the nodes $(X_i, X_{i+1}...X_j)$, for $i\le j$. Suppose the domain $D_t$ for $X_t$ is a compact subset of $\mathbb{R}^{d_t}$ for $t=1,2...T$. We use the $l_q$ ($q\ge 1$) Euclidean metric for all the Euclidean spaces with different dimensions. 

Let $\{P_n^{X_{i:j}}\}_{n=1}^\infty$ and $P^{X_{i:j}}$ be probability measures of $X_{i:j}$ for $i\le j$, $P_n^{X_i|X_j}$ and $P^{X_i|X_j}$ be the corresponding probability transition kernels. If $P_n^{X_{t:t+1}}$ converges to $P^{X_{t:t+1}}$ weakly for all $t = 1,2...T-1$, $P_n^{X_t|X_{t+1}}$ and $P^{X_{t+2}|X_{t+1}}$ are $C$-Lipschitz continuous in Wasserstein-$p$ metric for all $t=1,2...T-2$ and $n$, where $C$ is a constant, then $P_n^{X_{1:T}}$ converges to $P^{X_{1:T}}$ weakly.

\end{lemma}

\begin{proof}
We start from the case $T = 3$. For simplicity, we use $X, Y, Z$ to denote the nodes in sequence. 

For any $f(x,y,z)$ $K$-Lipschitz and bounded, we have
\begin{equation}
\begin{aligned}
    &\int f(x,y,z)dP_n^{XYZ}(x,y,z) \\
   =& \int \left(\int f(x,y,z)dP_n^{Z|Y=y}(z) \right) dP_{n}^{XY}(x,y)\\
   =& \int ( g_n(x,y) - g(x,y)) dP_n^{XY}(x,y)\\
   +& \int g(x,y)dP_n^{XY}(x,y), 
\end{aligned}
\end{equation}
where 
\begin{equation}
\begin{aligned}
    g(x,y) &=  \int f(x,y,z)dP^{Z|Y=y}(z)\\
    g_n(x,y) &=  \int f(x,y,z)dP_n^{Z|Y=y}(z).
\end{aligned}
\end{equation}

From Lemma~\ref{integralLip}, since $f(x,y,z)$ is Lipschitz, $P^{Z|Y=y}$ is Lipschitz in Wasserstein-$p$ sense, we have $g(x,y)$ is Lipschitz. $g(x,y)$ is also bounded since $f(x,y,z)$ is bounded. So we have
\begin{equation}
\begin{aligned}
   \int g(x,y)dP_n^{XY}(x,y) \rightarrow  \int g(x,y)dP^{XY}(x,y),
\end{aligned}
\end{equation}
since $P_n^{XY}$ converge to $P^{XY}$ weakly.

We then need to show $\int ( g_n(x,y) - g(x,y) )dP_n^{XY}(x,y)$ converges to 0. We prove by contradiction. Suppose it does not converge, then there exists $\epsilon>0$ and a subsequence of $n$ (denote as $i$) such that 
\begin{equation} \label{eqn:1}
\begin{aligned}
  \left| \int ( g_i(x,y) - g(x,y) )dP_i^{XY}(x,y) \right| \ge \epsilon, 
  \forall i.
\end{aligned}
\end{equation}
Without loss of generality, we assume that 
\begin{equation}
\begin{aligned}
  \int ( g_i(x,y) - g(x,y) )dP_i^{XY}(x,y) \ge \epsilon > 0, 
  \forall i,
\end{aligned}
\end{equation}
i.e.
\begin{equation}
\begin{aligned}
   & \iiint f(x,y,z) dP_i^{Z|Y=y}(z) dP_i^{X|Y=y}(x) dP_i^Y(y) \\
   - &\iiint f(x,y,z) dP^{Z|Y=y}(z)  dP_i^{X|Y=y}(x) dP_i^Y(y) \ge \epsilon > 0,  \forall i.
\end{aligned}
\end{equation} 

Let
\begin{equation}
\begin{aligned}
   & h_i(y,z) = \int f(x,y,z) dP_i^{X|Y=y}(x),
\end{aligned}
\end{equation} 
then
\begin{equation}\label{eqn:contradict1}
\begin{aligned}
   & \iint h_i(y,z) dP_i^{Z|Y=y}(z) dP_i^Y(y) \\
   - &\iint h_i(y,z) dP^{Z|Y=y}(z)  dP_i^Y(y) \ge \epsilon > 0,  \forall i.
\end{aligned}
\end{equation} 

Note that $f$ is $K$-Lipschitz, $P_i^{X|Y=y}$ is $K$-Lipschitz, we have $h_i$ are all $K(C+1)$-Lipschitz, thus uniformly equicontinuous. Also $h_i(y,z)$ are uniformly bounded by the bound of $f$, and the domain for $(Y,Z)$ is compact (since the domain $D_t$ for $X_t$ is compact for all $t$). By the Arzela-Ascoli theorem, there exists a subsequence $j$ such that $h_j \rightarrow h$ uniformly, where $h$ is $K(C+1)$-Lipschitz and bounded.
Therefore,
\begin{equation}
\begin{aligned}
   & \iint (h_j(y,z) - h(y,z)) dP_j^{Z|Y=y}(z) dP_j^Y(y) \\
   - &\iint (h_j(y,z) - h(y,z))  dP^{Z|Y=y}(z)  dP_j^Y(y) \rightarrow 0.
\end{aligned}
\end{equation} 
From Lemma~\ref{longlemma}
\begin{equation}
\begin{aligned}
   & \iint  h(y,z) dP_j^{Z|Y=y}(z) dP_j^Y(y) \\
   - &\iint h(y,z)  dP^{Z|Y=y}(z)  dP_j^Y(y) \rightarrow 0.
\end{aligned}
\end{equation} 
We then have 
\begin{equation}\label{eqn:contradict2}
\begin{aligned}
   & \iint  h_j(y,z) dP_j^{Z|Y=y}(z) dP_j^Y(y) \\
   - &\iint h_j(y,z)  dP^{Z|Y=y}(z)  dP_j^Y(y) \rightarrow 0.
\end{aligned}
\end{equation} 
We have a contradiction between Equation~\ref{eqn:contradict1} and~\ref{eqn:contradict2}. We finish the proof for $T=3$.

We then prove the case for general $T$ by induction. Suppose it holds for $T = N$, we now prove it for $T = N+1$.

From the conditions for the case $T=N+1$ and that the theorem holds for $T = N$, we have $P_n^{X_{1:N}}$ converges to $P^{X_{1:N}}$ weakly and $P_n^{X_{2:N+1}}$ converges to $P^{X_{2:N+1}}$ weakly. We now view $X_1$ as $X$, view the Cartesian product of $X_{2:N}$ as $Y$, and view $X_{N+1}$ as $Z$, so we have $P_n^{XY}$ converges to $P^{XY}$ weakly and $P_n^{YZ}$ converges to $P^{YZ}$ weakly. Also $P_n^{X|Y}$ and $P^{Z|Y}$  are $C$-Lipschitz continuious and Wasserstein-$p$ metric for all $n$, since $P_n^{X|Y} = P_n^{X_1|X_2}$ and $P^{Z|Y} = P^{X_{N+1}|X_{N}}$ from the Markovian property, and that $\|X_2\| \le \|X_{2:N}\|$, $\|X_{N}\| \le \|X_{2:N}\|$. Therefore, from the theorem for $T=3$ case we have that $P_n^{XYZ}$ converges to $P^{XYZ}$ weakly, i.e., $P_n^{X_{1:N+1}}$ converges to $P^{X_{1:N+1}}$ weakly.

\end{proof}

\subsection*{S6. A Counterexample in Continuous Sample Space}\label{sec:counterexample}

As a direct implementation of Corollary 7 in \cite{ding2020subadditivity}, for the Markov chain $X_{1:T}$ in a finite discrete sample space, $P_n^{X_{t:t+1}}$ converges to $P^{X_{t:t+1}}$ in Wasserstein-$p$ sense for each $t=1,2...T-1$ implies that $P_n^{X_{1:T}}$ converges to $P^{X_{1:T}}$ in Wasserstein-$p$ sense. However, if the Markov chains are defined in the continuous sample space, the implementation is, in general, not correct without further assumptions, e.g., the assumption of continuity for probability transition kernels. In this section, we provide a counterexample to show that.

We consider the Markov chain with $T=3$ and use $X, Y, Z$ to denote the nodes in sequence. We define $P_n^{XYZ}$ as follows:
\begin{equation}
\begin{aligned}
P_n^{XYZ}(0,0,0) &= \frac{1}{2}, \\
P_n^{XYZ}(1,\frac{1}{n},1) &= \frac{1}{2},
\end{aligned}
\end{equation}
and $P^{XYZ}$ as following:
\begin{equation}
\begin{aligned}
P^{XYZ}(0,0,0) &= \frac{1}{4}, \\
P^{XYZ}(0,0,1) &= \frac{1}{4}, \\
P^{XYZ}(1,0,0) &= \frac{1}{4}, \\
P^{XYZ}(1,0,1) &= \frac{1}{4}.
\end{aligned}
\end{equation}
We can easily check that $P_n^{XY}$ converges to $P^{XY}$ in Wasserstein-$p$ metric, since
\begin{equation}
\begin{aligned}
P_n^{XY}(0,0) &= \frac{1}{2}, \\
P_n^{XY}(1,\frac{1}{n}) &= \frac{1}{2},
\end{aligned}
\end{equation}
and
\begin{equation}
\begin{aligned}
P^{XY}(0,0) &= \frac{1}{2},\\
P^{XY}(1,0) &= \frac{1}{2},\\
\end{aligned}
\end{equation}
Similarly $P_n^{YZ}$ converges to $P^{YZ}$ in Wasserstein-$p$ metric.
However, $P_n^{XYZ}$ does not converge to $P^{XYZ}$ in Wasserstein-$p$ metric: the support of $P_n^{XYZ}$ and $(0,0,1)$ always have a distance larger than 1.

\bibliographystyle{unsrt}
\bibliography{cite}

\begin{thebibliography}{10}

\bibitem{hochreiter1997long}
Sepp Hochreiter and J{\"u}rgen Schmidhuber.
\newblock Long short-term memory.
\newblock {\em Neural computation}, 9(8):1735--1780, 1997.

\bibitem{chen2018neural}
Tian~Qi Chen, Yulia Rubanova, Jesse Bettencourt, and David~K Duvenaud.
\newblock Neural ordinary differential equations.
\newblock In {\em Advances in Neural Information Processing Systems}, pages
  6571--6583, 2018.

\bibitem{deshpande2018generative}
Ishan Deshpande, Ziyu Zhang, and Alexander~G Schwing.
\newblock Generative modeling using the sliced {W}asserstein distance.
\newblock In {\em Proceedings of the IEEE Conference on Computer Vision and
  Pattern Recognition}, pages 3483--3491, 2018.

\bibitem{goodfellow2014generative}
Ian Goodfellow, Jean Pouget-Abadie, Mehdi Mirza, Bing Xu, David Warde-Farley,
  Sherjil Ozair, Aaron Courville, and Yoshua Bengio.
\newblock Generative adversarial nets.
\newblock In {\em Advances in Neural Information Processing Systems}, pages
  2672--2680, 2014.

\bibitem{gulrajani2017improved}
Ishaan Gulrajani, Faruk Ahmed, Martin Arjovsky, Vincent Dumoulin, and Aaron~C
  Courville.
\newblock Improved training of {W}asserstein {GAN}s.
\newblock In {\em Advances in Neural Information Processing Systems}, pages
  5767--5777, 2017.

\bibitem{donoho2000high}
David~L Donoho et~al.
\newblock High-dimensional data analysis: The curses and blessings of
  dimensionality.
\newblock {\em AMS math challenges lecture}, 1(2000):32, 2000.

\bibitem{ambrosio2008gradient}
Luigi Ambrosio, Nicola Gigli, and Giuseppe Savar{\'e}.
\newblock {\em Gradient flows: in metric spaces and in the space of probability
  measures}.
\newblock Springer Science \& Business Media, 2008.

\bibitem{ambrosio2008hamiltonian}
Luigi Ambrosio and Wilfrid Gangbo.
\newblock Hamiltonian {ODE}s in the {W}asserstein space of probability
  measures.
\newblock {\em Communications on Pure and Applied Mathematics: A Journal Issued
  by the Courant Institute of Mathematical Sciences}, 61(1):18--53, 2008.

\bibitem{li2020scalable}
Xuechen Li, Ting-Kam~Leonard Wong, Ricky~TQ Chen, and David Duvenaud.
\newblock Scalable gradients for stochastic differential equations.
\newblock {\em arXiv preprint arXiv:2001.01328}, 2020.

\bibitem{jia2019neural}
Junteng Jia and Austin~R Benson.
\newblock Neural jump stochastic differential equations.
\newblock In {\em Advances in Neural Information Processing Systems}, pages
  9843--9854, 2019.

\bibitem{liu2019neural}
Xuanqing Liu, Tesi Xiao, Si~Si, Qin Cao, Sanjiv Kumar, and Cho-Jui Hsieh.
\newblock Neural {SDE}: Stabilizing neural {ODE} networks with stochastic
  noise.
\newblock {\em arXiv preprint arXiv:1906.02355}, 2019.

\bibitem{tzen2019neural}
Belinda Tzen and Maxim Raginsky.
\newblock Neural stochastic differential equations: Deep latent {Gaussian}
  models in the diffusion limit.
\newblock {\em arXiv preprint arXiv:1905.09883}, 2019.

\bibitem{tzen2019theoretical}
Belinda Tzen and Maxim Raginsky.
\newblock Theoretical guarantees for sampling and inference in generative
  models with latent diffusions.
\newblock {\em arXiv preprint arXiv:1903.01608}, 2019.

\bibitem{ruthotto2020machine}
Lars Ruthotto, Stanley~J Osher, Wuchen Li, Levon Nurbekyan, and Samy~Wu Fung.
\newblock A machine learning framework for solving high-dimensional mean field
  game and mean field control problems.
\newblock {\em Proceedings of the National Academy of Sciences},
  117(17):9183--9193, 2020.

\bibitem{yang2020physics}
Liu Yang, Dongkun Zhang, and George~Em Karniadakis.
\newblock Physics-informed generative adversarial networks for stochastic
  differential equations.
\newblock {\em SIAM Journal on Scientific Computing}, 42(1):A292--A317, 2020.

\bibitem{liu2020rode}
Junyu Liu, Zichao Long, Ranran Wang, Jie Sun, and Bin Dong.
\newblock {RODE-Net}: Learning ordinary differential equations with randomness
  from data.
\newblock {\em arXiv preprint arXiv:2006.02377}, 2020.

\bibitem{elerian2001likelihood}
Ola Elerian, Siddhartha Chib, and Neil Shephard.
\newblock Likelihood inference for discretely observed nonlinear diffusions.
\newblock {\em Econometrica}, 69(4):959--993, 2001.

\bibitem{eraker2001mcmc}
Bj{\o}rn Eraker.
\newblock {MCMC} analysis of diffusion models with application to finance.
\newblock {\em Journal of Business \& Economic Statistics}, 19(2):177--191,
  2001.

\bibitem{sarkka2015posterior}
Simo S{\"a}rkk{\"a}, Jouni Hartikainen, Isambi~Sailon Mbalawata, and Heikki
  Haario.
\newblock Posterior inference on parameters of stochastic differential
  equations via non-linear {Gaussian} filtering and adaptive {MCMC}.
\newblock {\em Statistics and Computing}, 25(2):427--437, 2015.

\bibitem{archambeau2007gaussian}
Cedric Archambeau, Dan Cornford, Manfred Opper, and John Shawe-Taylor.
\newblock Gaussian process approximations of stochastic differential equations.
\newblock {\em Journal of Machine Learning Research}, 1:1--16, 2007.

\bibitem{vrettas2015variational}
Michail~D Vrettas, Manfred Opper, and Dan Cornford.
\newblock Variational mean-field algorithm for efficient inference in large
  systems of stochastic differential equations.
\newblock {\em Physical Review E}, 91(1):012148, 2015.

\bibitem{bakarji2021data}
Joseph Bakarji and Daniel~M Tartakovsky.
\newblock Data-driven discovery of coarse-grained equations.
\newblock {\em Journal of Computational Physics}, page 110219, 2021.

\bibitem{mao2019nonlocal}
Zhiping Mao, Zhen Li, and George~Em Karniadakis.
\newblock Nonlocal flocking dynamics: Learning the fractional order of pdes
  from particle simulations.
\newblock {\em Communications on Applied Mathematics and Computation},
  1(4):597--619, 2019.

\bibitem{mavridis2020learning}
Christos~N Mavridis, Amoolya Tirumalai, and John~S Baras.
\newblock Learning interaction dynamics from particle trajectories and density
  evolution.
\newblock In {\em 2020 59th IEEE Conference on Decision and Control (CDC)},
  pages 1014--1019. IEEE, 2020.

\bibitem{cucker2007emergent}
Felipe Cucker and Steve Smale.
\newblock Emergent behavior in flocks.
\newblock {\em IEEE Transactions on automatic control}, 52(5):852--862, 2007.

\bibitem{santambrogio2017euclidean}
Filippo Santambrogio.
\newblock $\{$Euclidean, metric, and Wasserstein$\}$ gradient flows: an
  overview.
\newblock {\em Bulletin of Mathematical Sciences}, 7(1):87--154, 2017.

\bibitem{maas2013rectifier}
Andrew~L Maas, Awni~Y Hannun, and Andrew~Y Ng.
\newblock Rectifier nonlinearities improve neural network acoustic models.
\newblock In {\em Proceedings of the International Conference on Machine
  Learning}, volume~30, page~3, 2013.

\bibitem{kingma2014adam}
Diederik~P Kingma and Jimmy Ba.
\newblock Adam: A method for stochastic optimization.
\newblock {\em arXiv preprint arXiv:1412.6980}, 2014.

\bibitem{ding2020subadditivity}
Mucong Ding, Constantinos Daskalakis, and Soheil Feizi.
\newblock Subadditivity of probability divergences on {Bayes-Nets} with
  applications to time series {GANs}.
\newblock {\em CoRR}, abs/2003.00652, 2020.

\bibitem{daskalakis2017training}
Constantinos Daskalakis, Andrew Ilyas, Vasilis Syrgkanis, and Haoyang Zeng.
\newblock Training {GAN}s with optimism.
\newblock In {\em International Conference on Learning Representations}, 2018.

\bibitem{villani2008optimal}
C{\'e}dric Villani.
\newblock {\em Optimal transport: old and new}, volume 338.
\newblock Springer Science \& Business Media, 2008.

\bibitem{klenke2013probability}
Achim Klenke.
\newblock {\em Probability theory: a comprehensive course}.
\newblock Springer Science \& Business Media, 2013.

\end{thebibliography}

\end{document}